\documentclass[10pt]{article}
\usepackage{geometry} 
\usepackage{algorithm, algorithmic,hyperref}

\usepackage{microtype}
\usepackage{graphicx}
\usepackage{subfigure}
\usepackage{booktabs} 
 \usepackage{enumitem}
 \usepackage{nicefrac}

\usepackage{authblk}

\pdfoutput=1

\usepackage{amsfonts}
\usepackage{amsmath}

\usepackage{mdframed} 
\usepackage{thmtools}
\usepackage{mathtools}

\declaretheorem[within=section]{definition}
\declaretheorem[sibling=definition]{theorem}
\declaretheorem[sibling=definition]{proposition}
\declaretheorem[within=section]{assumption}

\declaretheorem[sibling=definition]{remark}

\declaretheorem[sibling=definition]{lemma}

\usepackage{colordvi,epsfig}
\usepackage[table]{xcolor}   
\usepackage{verbatim}

\usepackage{longtable}

\usepackage{verbatim,tikz}
\usepackage{caption}   

\DeclareMathOperator*{\argmin}{argmin}

\graphicspath{{./figures/}}  

\newcommand{\ignore}[1]{}
\newcommand{\R}{\mathbb{R}}

\newcommand{\prox}{\mathop{\mathrm{prox}}\nolimits}

\newcommand{\eqdef}{\coloneqq}

\newcommand{\Span}{{\rm Span}} 

\newcommand{\cA}{{\cal A}}

\newcommand{\cL}{{\cal L}}

\newcommand{\cO}{{\cal O}}

\newcommand{\cS}{{\cal S}}

\newcommand{\mI}{{\bf I}}

\newcommand{\mM}{{\bf M}}

\newcommand{\mQ}{{\bf Q}_r}

\newcommand{\ones}{e}

\newcommand{\E}[1]{\mathbb{E}\left[#1\right] }

\newcommand{\norm}[1]{\left \| #1 \right\|}

\usepackage[colorinlistoftodos,bordercolor=orange,backgroundcolor=orange!20,linecolor=orange,textsize=scriptsize]{todonotes}

\usepackage{hyperref}

\def\<#1,#2>{\left\langle #1,#2\right\rangle}

\newcommand{\diam}{{\mathrm{Diam}}}

\newcommand{\probx}{ \rho}
\newcommand{\proby}{p}

\newcommand{\tR}{{i}} 
\newcommand{\TR}{{n}} 

\newcommand{\Lloc}{{\tilde{L}}} 

\newcommand{\mat}{\begin{pmatrix}
		c& -c\\
		-c & c
\end{pmatrix} }

\newcommand{\mmat}{\begin{pmatrix}
		\tfrac{(\nhalf+1)c}{\nhalf} & -\tfrac{(\nhalf+1)c}{\nhalf}\\
		-\tfrac{(\nhalf+1)c}{\nhalf} & \tfrac{(\nhalf+1)c}{\nhalf}
\end{pmatrix} }

\newcommand{\nhalf}{M }
\newcommand{\cLL}{\cL }
\newcommand{\LLL}{L}
\newcommand{\flocc}{{ \tilde{f}}}

\newcommand{\ggg}{g }

\newcommand{\locf}{\zeta}

\newcommand{\comm}{C}

\usepackage{pifont}
\newcommand{\cmark}{{\color{green}\ding{51}}}%
\newcommand{\xmark}{{\color{red} \ding{55}}}%

\newcommand*{\QED}[1][$\square$]{%
	\leavevmode\unskip\penalty9999 \hbox{}\nobreak\hfill
	\quad\hbox{#1}%
}

\title{Lower Bounds and Optimal Algorithms for \\  Personalized Federated Learning}

\author{Filip Hanzely}
\author{Slavom\'ir Hanzely}
\author{Samuel Horv\'ath}
\author{Peter Richt\'{a}rik}

\affil{King Abdullah University of Science and Technology\\ Thuwal, Saudi Arabia}

\date{June 3, 2020\footnote{The paper was submitted on June 3, 2020. Only a minor edits were made after that date.}}

\begin{document}
	
	\maketitle
	\begin{abstract}
	In this work, we consider the optimization formulation of personalized federated learning recently introduced by~\cite{hanzely2020federated} which was shown to give an alternative explanation to the
workings of local {\tt SGD} methods. Our first contribution is  establishing the first lower bounds for this formulation, for both the communication complexity and the local oracle complexity.
Our second contribution is the design of several optimal methods matching these lower bounds in almost all regimes. These are the first provably optimal methods for personalized federated learning.
Our optimal methods include an accelerated variant of {\tt FedProx}, and an accelerated variance-reduced version of {\tt FedAvg}/Local {\tt SGD}. We demonstrate the practical superiority of our methods through
extensive numerical experiments.
	\end{abstract}

	\section{Introduction}

	Federated Learning (FL)~\cite{mcmahan17a,konevcny2016federated} is a relatively new field that attracted much attention recently. Specifically, FL is a subset of distributed machine learning that aims to fit the data stored locally on plentiful clients. Unlike typical distributed learning inside a data center, each client only sees his/her data, which might differ from the population average significantly. Furthermore, as the clients are often physically located far away from the central server, communication becomes a notable bottleneck, which is far more significant compared to in-datacenter learning.

	While the main difference between FL and the rest of the machine learning lies in means of the training, the two scenarios are often identical from the modeling perspective. In particular, the standard FL aims to find the minimizer of the overall population loss:
	
	\begin{align}\label{eq:fl_standard}
	\min_{z\in \R^d} \frac1n \sum_{i=1}^n f_i(z) = & \min_{x_1, x_2, \dots, x_n \in \R^d} \, \frac1n \sum_{i=1}^n f_i(x_i),  \\ \nonumber
	&\, \text{s. t. }\,  x_1 = x_2 = \dots = x_n
	\end{align}
	where $f_i$ is the loss of the client $i$ that only depends on his/her own local data. 
	
	However, there is major criticism of the objective~\eqref{eq:fl_standard} for many of the FL applications~\cite{wu2020personalized, kulkarni2020survey, deng2020adaptive}. Specifically, the minimizer of the overall population loss might not be the ideal model for a given client, given that his/her data distribution differs from the population significantly. A good example to illustrate the requirement of personalized FL models is the prediction of the next word written on a mobile keyboard, where a personalized FL approach~\cite{hard2018federated} significantly outperformed the non-personalized one.
	
There are multiple strategies in the literature for incorporating the personalization into FL: multi-task learning~\cite{vanhaesebrouck2016decentralized, smith2017federated, fallah2020personalized}, transfer learning~\cite{zhao2018federated, khodak2019adaptive}, variational inference~\cite{corinzia2019variational}, mixing of the local and global models~\cite{peterson2019private, hanzely2020federated, mansour2020three, deng2020adaptive} and others~\cite{eichner2019semi}. See also~\cite{kulkarni2020survey, FL-big} for a personalized FL survey.
	
	In this work, we focus on the mixing FL objective from~\cite{hanzely2020federated} which is well-known from the area of distributed optimizaton~\cite{lan2018communication, gorbunov2019optimal} and distributed transfer learning~\cite{liu2017distributed, wang2018distributed}. The mentioned formulation allows the local models $x_i$ to be mutually different, while penalizing their dissimilarity:
	
	\begin{equation}\label{eq:main}
	\min_{x = [x_1,\dots,x_n]\in \R^{nd}, \forall i:\, x_i\in \R^d} \left\{ F(x) \eqdef  \underbrace{\tfrac{1}{n}\sum \limits_{i=1}^n  f_i(x_i)}_{\eqdef f(x)} + \lambda \underbrace{\tfrac{1}{2 n}\sum \limits_{i=1}^n \norm{x_i-\bar{x}}^2}_{\eqdef  \psi(x)} \right\}
	\end{equation}

	Suprisingly enough, the optimal solution $x^\star = [x_1^\star, x_2^\star, \dots, x_n^\star] \in \R^{nd}$ of~\eqref{eq:main} can be expressed as $x_i^\star = \bar{x}^\star - \frac{1}{\lambda}\nabla f_i(x_i^\star)$, where $\bar{x}^\star  = \frac1n\sum_{i=1}^n x^\star_i$~\cite{hanzely2020federated}, which strongly resembles the famous {\tt MAML}~\cite{finn2017model}.

	In addition to personalization, the above formulation sheds light on the most prominent FL optimizer -- local {\tt SGD}/{\tt FedAvg}\cite{mcmahan2016communication}. Specifically, it was shown that a simple version of Stochastic Gradient Descent ({\tt SGD}) applied on~\eqref{eq:main} is essentially\footnote{Up to the stepsize and random number of the local gradient steps.} equivalent to {\tt FedAvg} algorithm~\cite{hanzely2020federated}. Furthermore, the FL formulation~\eqref{eq:main} enabled local gradient methods to outperform their non-local cousins when applied to heterogeneous data problems.\footnote{Surprisingly enough, the non-local algorithms outperform their local counterparts when applied to solve the classical FL formulation~\eqref{eq:fl_standard} with heterogeneous data.}

	\section{Contributions}
	In this paper, we study the personalized FL formulation~\eqref{eq:main}. We propose a lower complexity bounds for communication and local computation, and develop several algorithms capable of achieving it. Our contributions can be listed as follows:
	
	\noindent $\bullet$ We propose a \emph{lower bound on the communication complexity} of the federated learning formulation~\eqref{eq:main}. We show that for any algorithm that satisfies a certain reasonable assumption (see As.~\ref{as:oracle}) there is an instance of~\eqref{eq:main} with $L$-smooth, $\mu$-strongly convex\footnote{We say that function $h:\R^d \rightarrow \R$ is $L$-smooth if for each $z,z' \in \R^d $ we have $h(z) \leq h(z') + \langle \nabla h(z), z'-z\rangle + \tfrac  L 2 \norm{z-z'}^2.$ Similarly, a function $h:\R^d \rightarrow \R$ is $\mu$-strongly convex, if for each $z,z' \in \R^d $ it holds $h(z) \geq h(z') + \langle \nabla h(z), z'-z\rangle + \tfrac  \mu 2 \norm{z-z'}^2.$} local objectives $f_i$ requiring at least $\cO\left( \sqrt{\tfrac{\min \{ L, \lambda\}}{\mu}}\log \tfrac1\varepsilon\right)$ communication rounds to get to the $\varepsilon$-neighborhood of the optimum.
	
	\noindent $\bullet$ We investigate the \emph{lower complexity bound on the number of local oracle calls}. We show that one requires at least $\cO\left( \sqrt{\tfrac{\min \{ L, \lambda\}}{\mu}}\log \tfrac1\varepsilon\right)$ proximal oracle calls\footnote{Local proximal oracle reveals $\{ \prox_{\beta f}(x), \nabla f(x)\}$ for any $x\in \R^{nd}, \beta>0$. Local gradient oracle reveals $\{ \nabla f(x)\}$ for any $x\in \R^{nd}$.} or at least $\cO\left( \sqrt{\tfrac{L}{\mu}}\log \tfrac1\varepsilon\right)$  evaluations of local gradients. Similarly, given that each of the local objectives is of a $m$-finite-sum structure with $\Lloc$-smooth summands, we show that at least $\cO\left(\left(m +  \sqrt{\tfrac{m\Lloc}{\mu}}\right)\log \tfrac1\varepsilon\right)$ gradients of the local summands are required.

	\noindent $\bullet$ We discuss several approaches to solve~\eqref{eq:main} which achieve the \emph{optimal communication complexity and optimal local gradient complexity} under various circumstances. Specializing the approach from~\cite{wang2018distributed} to our problem, we apply Accelerated Proximal Gradient Descent ({\tt APGD}) in two different ways -- either we take a gradient step with respect to $f$ and proximal step with respect to $\lambda \psi$ or vice versa. In the first case, we get both the communication complexity and local gradient complexity of the order $\cO\left( \sqrt{\tfrac{L}{\mu}}\log \tfrac1\varepsilon\right)$ which is optimal if $L\leq \lambda$. In the second case, we get both the communication complexity and the local prox complexity of the order $\cO\left( \sqrt{\tfrac{\lambda}{\mu}}\log \tfrac1\varepsilon\right)$, thus optimal if $L\geq \lambda$. Motivated again by~\cite{wang2018distributed}, we argue that local prox steps can be evaluated inexactly\footnote{Such an approach was already considered in~\cite{li2018federated, pathak2020fedsplit} for the standard FL formulation~\eqref{eq:fl_standard}.} either by running locally Accelerated Gradient Descent ({\tt AGD})~\cite{nesterov1983method} or {\tt Katyusha}~\cite{allen2017katyusha} given that the local objective is of a $m$-finite sum structure with $\Lloc$-smooth summands. Local {\tt AGD} approach preserves $\cO\left(\sqrt{\tfrac{\lambda}{\mu}}\log \tfrac1\varepsilon \right)$ communication complexity and yields $\tilde{\cO}\left(\sqrt{\tfrac{L+\lambda}{\mu}} \right)$ local gradient complexity, both of them optimal for $L\geq \lambda$ (up to log factors). Similarly, employing {\tt Katyusha} locally, we obtain the communication complexity of order $\cO\left(\sqrt{\tfrac{\lambda}{\mu}}\log \tfrac1\varepsilon\right)$ and the local gradient complexity of order $\tilde{\cO}\left(m\sqrt{\tfrac{\lambda}{\mu}} + \sqrt{m\tfrac{\Lloc}{\mu}} \right) $; the former is optimal once $L\geq \lambda$, while the latter is (up to $\log$ factor) optimal once $  m\lambda \leq  \Lloc$.

	\noindent $\bullet$ The inexact {\tt APGD} with local randomized solver has three drawbacks: (i) there are extra $\log$ factors in the local gradient complexity, (ii) boundedness of the algorithm iterates as an assumption is required and (iii) the communication complexity is suboptimal for $\lambda>L$. In order to fix all the issues, \emph{we accelerate the {\tt L2SGD+} algorithm} from~\cite{hanzely2020federated}. The proposed algorithm, {\tt AL2SGD+}, enjoys the optimal communication complexity $\cO\left(  \sqrt{\tfrac{ \min \{ \Lloc, \lambda\}}{\mu}}\log\tfrac1\varepsilon\right)$ and the local summand gradient complexity $ \cO\left(
	\left(m +\sqrt{\tfrac{m( \Lloc +  \lambda)}{\mu}}\right)\log\tfrac1\varepsilon
	\right)$, which is optimal for $\lambda \leq \Lloc$. Unfortunately, the two bounds are not achieved at the same time, as we shall see.

\noindent $\bullet$ As a consequence of all aforementioned points, {\bf we show the optimality of local algorithms applied on FL problem~\eqref{eq:main} with heterogeneous data}. We believe this is an important piece that was missing in the literature. Until now, the local algorithms were known to be optimal only when all nodes own an identical set of data, which is questionable for the FL applications. By showing the optimality of local methods, we justify the standard FL practices (i.e., using local methods in the practical scenarios with non-iid data).

	Table~\ref{tbl:algs2} presents a summary of the described results: for each algorithm, it indicates the local oracle requirement and the circumstances under which the corresponding complexities are optimal.
	
	\begin{table}[!t]
		\begin{center}
			\small
			\begin{tabular}{|c|c|c|c|}
				\hline
				{\bf Algorithm}  & {\bf Local oracle}& { \bf  Optimal \# comm}&  { \bf  Optimal \# local }    \\
				\hline
				\hline
				{\tt L2GD}~\cite{hanzely2020federated} &  Grad & \xmark & \xmark \\
				\hline
				{\tt L2SGD+}~\cite{hanzely2020federated} &  Stoch grad &  \xmark & \xmark  \\
				\hline
				{\tt APGD1}~\cite{wang2018distributed}  (A.~\ref{alg:fista}) &  Prox & \cmark \, (if $\lambda \leq L $)&  \cmark  \, (if $\lambda \leq L $) \\
				\hline
				{\tt APGD2}~\cite{wang2018distributed} (A.~\ref{alg:fista_2}) &  Grad & \cmark \, (if $\lambda \geq L $)&  \cmark  \\
				\hline
				{\tt APGD2}~\cite{wang2018distributed} (A.~\ref{alg:fista_2}) &  Stoch grad & \cmark \, (if $\lambda \geq L $)&  \xmark  \\
				\hline
				{\tt IAPGD}~\cite{wang2018distributed} (A.~\ref{alg:fista_inex})  + {\tt AGD}~\cite{nesterov1983method} &  Grad & \cmark \, (if $\lambda \leq L $)&  \cmark  \, (if $\lambda \leq L $)
				\\
				\hline
				{\tt IAPGD}~\cite{wang2018distributed} (A.~\ref{alg:fista_inex})  + {\tt Katyusha}~\cite{allen2017katyusha}
				&  Stoch grad & \cmark \, (if $\lambda \leq L $)  & \cmark \, (if $m\lambda \leq \Lloc $)
				\\
				\hline
				{\tt AL2SGD+}
				{(A.~\ref{alg:acc_stoch})}
				&  Stoch grad &  \cmark  &   \cmark \, $\left( \text{if } \lambda \leq \Lloc \right)$ \\
				\hline
			\end{tabular}
		\end{center}
		\caption{Algorithms for solving~\eqref{eq:main} and their (optimal) complexities.
		}
		\label{tbl:algs2}
	\end{table}

	\paragraph{Optimality.} Next we present Table~\ref{tbl:optimal} which carries an information orthogonal to Table~\ref{tbl:algs2}. In particular, Table~\ref{tbl:optimal} indicates whether our lower and upper complexities match for a given pair of \{local oracle, type of complexity\}. The lower and upper complexity bounds on the number of communication rounds match regardless of the local oracle. 
	Similarly, the local oracle calls match almost always with one exception when the local oracle provides summand gradients and $\lambda>\Lloc$. 
	
	\begin{remark}
		Our upper and lower bounds do not match for the local summand gradient oracle once we are in the classical FL setup~\eqref{eq:main}, which we recover for $\lambda=\infty$. In such a case, an optimal algorithm was developed only very recently~\cite{hendrikx2020optimal} under a slightly stronger oracle -- the proximal oracle for the local summands. 
	\end{remark}

	\begin{table}[!t]
		\begin{center}
			\small
			\begin{tabular}{|c|c|c|c|}
				\hline
				{\bf Local oracle}& \begin{tabular}{c} 
					{\bf Optimal}
					\\
					{ \bf  \# Comm }
				\end{tabular} 
				&   \begin{tabular}{c} 
					{\bf Optimal}
					\\
					{ \bf  \# Local calls }
				\end{tabular}&  {\bf Algorithm}     \\
				\hline
				\hline
				Proximal & \cmark &  \cmark  &
				{
					$
					\begin{cases}
					\lambda \geq L: &{\tt APGD2}~\text{\cite{wang2018distributed}} \text{(A.~\ref{alg:fista_2})}
					\\
					\lambda \leq L: &{\tt APGD1}~\text{\cite{wang2018distributed}} \text{(A.~\ref{alg:fista})}
					\end{cases}
					$
				}
				\\
				\hline
				Gradient &  \cmark &   \cmark &
				{
					$
					\begin{cases}
					\lambda \geq L:& {\tt APGD2}~\text{\cite{wang2018distributed}} \text{(A.~\ref{alg:fista_2})}
					\\
					\lambda \leq L:&  {\tt IAPGD}~\text{\cite{wang2018distributed}} \text{(A.~\ref{alg:fista_inex})}  + {\tt AGD}~\text{\cite{nesterov1983method}}
					\end{cases}$
				}
				\\
				\hline
				Stoch grad & \cmark  &  
				{ 
					$
					\text{\cmark} \,\,\, \, \text{if } m\lambda\leq \Lloc
					$
				} 
				&
				{ 
					$
					\begin{cases}
					\lambda \geq L:& {\tt APGD2}~\text{\cite{wang2018distributed}} \text{(A.~\ref{alg:fista_2}) }
					\\
					\lambda \leq L:&  {\tt IAPGD}~\text{\cite{wang2018distributed}} \text{(A.~\ref{alg:fista_inex})}  + {\tt Katyusha}~\text{\cite{allen2017katyusha}}
					\end{cases}$
				}
				\\
				\hline
				Stoch grad & 
				{ 
					\begin{tabular}{c l} 
						\cmark
						\\
						\xmark
					\end{tabular}
				} 
				&  
				{ 
					\begin{tabular}{c l} 
						\xmark &
						\\
						\cmark &  if  $\lambda\leq \Lloc$
					\end{tabular}
				} 
				&
				{\tt AL2SGD+}$^{(*)}$
				\\
				\hline
			\end{tabular}
		\end{center}
		\caption{Matching (up to $\log$ and constant factors) lower and upper complexity bounds for solving~\eqref{eq:main}. Indicator  \cmark\, means that the lower and upper bound are matching up to constant and log factors, while \xmark\, means the opposite. $^{(*)}$ ({\tt AL2SGD+} under stochastic gradient oracle): {\tt AL2SGD+} can be optimal either in terms of the communication or in terms of the local computation; the two cases require a slightly different parameter setup.
		}
		\label{tbl:optimal}
	\end{table}

	\section{Lower complexity bounds \label{sec:lower}}
	
	Before stating the lower complexity bounds for solving~\eqref{eq:main}, let us formalize the notion of an oracle that an algorithm interacts with.
	
	As we are interested in both communication and local computation, we will also distinguish between two different oracles: the communication oracle and the local oracle. While the communication oracle allows the optimization history to be shared among the clients, the local oracle $\text{Loc}(x_i,i)$ provides either a local proximal operator, local gradient, or local gradient of a summand given that a local loss is of a finite-sum structure itself $
	f_\tR(x_{\tR}) = \tfrac1m \sum_{j =1}^m \flocc_{i,j}(x_{\tR})$:

	\[
	\text{Loc}(x,i) = 
	\begin{cases}
	\{ \nabla f_i(x_i), \prox_{\beta_{i} {f_i}} (x_i)  \} & \text{if oracle is \emph{proximal} (for any $\beta_i\geq0$)}
	\\
	\{ \nabla f_i(x_i)  \} & \text{if oracle is \emph{gradient}}
	\\
	\{ \nabla \flocc_{i,j_i}(x_i)  \} & \text{if oracle is \emph{summand gradient} (for any $1\leq j_i\leq m$)}
	\end{cases}
	\]
	for all clients $i$ simultaneously, which we refer to as a single local oracle call.
	
	Next, we restrict ourselves to algorithms whose iterates lie in the span of previously observed oracle queries. Assumption~\ref{as:oracle} formalizes the mentioned notion. 
	
	\begin{assumption}\label{as:oracle}
		Let $\{x^k\}_{k=1}^\infty$ be iterates generated by algorithm $\cA$. For $1\leq i\leq n$ let $\{S_i^k\}_{k=0}^\infty$ be a sequence of sets
		defined recursively as follows:
		\begin{align*}
		S_i^{0}&=  \Span(x_i^0)\\
		S_i^{k+1} &= \begin{cases}
		\Span\left(S_i^k,  \mathrm{Loc}(x^k,i)\right)& \text{if  } \locf(k) = 1
		\\
		\Span\left( S_1^k, S_2^k, \dots, S_n^k\right) & \text{otherwise,}
		\end{cases}
		\end{align*}
		where $\locf(k)= 1$ if the local oracle was queried at the iteration $k$, otherwise $\locf(k)=0$. 
		Then, assume that $x_i^k\in S_i^k$ .
	\end{assumption}

Assumption~\ref{as:oracle} is rahter standard in the literature of distributed optimization~\cite{scaman2018optimal, hendrikx2020optimal}; it informally means that the iterates of $\cA$ lie in the span of explored directions only. A similar restriction is in place for several standard optimization lower complexity bounds~\cite{nesterov2018lectures, lan2018optimal}. We shall, however, note that Assumption~\ref{as:oracle} can be omitted by choosing the worst-case objective adversarially based on the algorithm decisions~\cite{nemirovsky1983problem, woodworth2016tight, woodworth2018graph}. We do not explore this direction for the sake of simplicity.

	\subsection{Lower complexity bounds on the communication}
	
	Next, we present the lower bound on the communication complexity of problem~\eqref{eq:main}.
	
	\begin{theorem}\label{thm:lb}
		
		Let $k\geq 0, L\geq \mu,  \lambda\geq \mu$. Then, there exist $L$-smooth $\mu$-strongly convex functions $f_1, f_2, \dots f_n: \R^d \rightarrow \R $ and a starting point $x^0\in \R^{nd}$, such that the sequence of iterates $\{x^t\}_{t=1}^k$ generated by any algorithm $\cA$ meeting Assumption~\ref{as:oracle} satisfies
		
		\begin{equation}\label{eq:lb}
		\| x^{k} - x^\star\|^2  \geq \tfrac 1 4 \left(1-10\max\left\{ \sqrt{\tfrac{\mu}{\lambda}}, \sqrt{\tfrac{\mu}{L-\mu}}\right \} \right)^{\comm(k)+1} \| x^0 - x^\star\|^2.
		\end{equation}
		
		Above, $\comm(k)$ stands for the number of communication oracle queries at the first $k$ iterations of $\cA$.
		
	\end{theorem}

	Theorem~\ref{thm:lb} shows that in order get to an $\varepsilon$ close to the optimum, one needs at least $\cO\left(\sqrt{\tfrac{\min\{ L, \lambda\}}{\mu}}\log \tfrac1\varepsilon \right)$ rounds of the communications. This reduces to known communication complexity $\cO\left(\sqrt{\tfrac{L}{\mu}}\log \tfrac1\varepsilon \right)$ for standard FL objective~\eqref{eq:fl_standard} from~\cite{scaman2018optimal, hendrikx2020optimal} when $\lambda=\infty$.\footnote{See also~\cite{woodworth2018graph} for a similar lower bound in a slightly different setup.}

	\subsection{Lower complexity bounds on the local computation\label{sec:lower_local}}
	
	Next, we present the lower complexity bounds on the number of the local oracle calls for three different types of a local oracle. In a special case when $\lambda=\infty$, we recover known local oracle bounds for the classical FL objective~\eqref{eq:fl_standard} from~\cite{hendrikx2020optimal}.
	
	{\bf Proximal oracle.} The construction from Theorem~\ref{thm:lb} not only requires $\cO\left(  \sqrt{\tfrac{\min\left\{\lambda, L\right\}}{\mu}} \log \tfrac1\varepsilon \right)$ communication rounds to reach $\varepsilon$-neighborhood of the optimum, it also requires at least $\cO\left(  \sqrt{\tfrac{\min\left\{\lambda, L\right\}}{\mu}} \log \tfrac1\varepsilon \right)$ calls of any local oracle, which serves as the lower bound on the local proximal oracle. 
	
	{\bf Gradient oracle.} 
	Setting $x^0= 0\in \R^{nd}$ and $f_1=f_2=\dots = f_n$, the problem~\eqref{eq:main} reduces to minimize a single local objective $f_1$. Selecting next $f_1$ as the worst-case quadratic function from~\cite{nesterov2018lectures}, the corresponding objective requires at least $\cO\left(\sqrt{\tfrac{L}{\mu}} \log\tfrac{1}{\varepsilon}\right)$ gradient calls to reach $\varepsilon$-neighborhood, which serves as our lower bound. Note that the parallelism does not help as the starting point is identical on all machines and the construction of $f$ only allows to explore a single coordinate per a local call, regardless of the communication.

	{\bf Summand gradient oracle.}
	Suppose that $\flocc_{i,j}$ is $\Lloc$-smooth for all $1\leq j\leq m, 1\leq i\leq n$. Let us restrict ourselves on a class of client-symmetric algorithms such that $x^{k+1}_i = \cA(H_i^k, H_{-i}^k, C^k)$, where $H_i$ is history of local gradients gathered by client $i$, $H_{-i}$ is an \emph{unordered} set with elements $H_l$ for all $l\neq i$ and $C^k$ are indices of the communication rounds of the past. We assume that $\cA$ is either deterministic, or generated from given seed that is identical for all clients initially.\footnote{We suspect that assuming the perfect symmetry across nodes is not necessary and can be omitted using more complex arguments. In fact, we believe that allowing for a varying scale of the local problem across the workers so that the condition number remains constant, we can adapt the approach from~\cite{hendrikx2020optimal} to obtain the desired local summand gradient complexity without assuming the symmetry.} Setting again $x^0= 0\in \R^{nd}$ and $f_1=f_2=\dots = f_n$, the described algorithm restriction yields $x_1^k=x_2^k= \dots = x_n^k$ for all $k\geq 0$. Consequently, the problem reduces to minimize a single finite sum objective $f_1$ which requires at least $\cO\left(m+ \sqrt{\tfrac{m\Lloc}{\mu}} \log\tfrac{1}{\varepsilon}\right)$ summand gradient calls~\cite{lan2018optimal, woodworth2016tight}.

	\section{Optimal algorithms \label{sec:upperbound}}
	In this section, we present several algorithms that match the lower complexity bound on the number of communication rounds and the local steps obtained in Section~\ref{sec:lower}.
	
	\subsection{Accelerated Proximal Gradient Descent ({\tt APGD}) for Federated Learning \label{sec:apgd_simple}}
	
	The first algorithm we mention is a version of the accelerated proximal gradient descent~\cite{beck2009fast}. In order to see how the method specializes in our setup, let us first describe the non-accelerated counterpart -- proximal gradient descent ({\tt PGD}).

	Let a function $h: \R^{nd}\rightarrow \R$ be $L_h$-smooth and $\mu_h$-strongly convex, and function $\phi: \R^{nd}\rightarrow \R \cup \{\infty\}$ be convex. In its most basic form, iterates of {\tt PGD} to minimize a regularized convex objective $h(x) + \phi(x)$ are generated recursively as follows
	\begin{equation} \label{eq:pgd}
	x^{k+1} = \prox_{\tfrac{1}{L_h }\phi}\left( x^k - \tfrac{1}{L_h} \nabla h(x^k)\right) = \argmin_{x\in \R^{nd}} \phi(x) - \frac{L_h}{2} \left\| x - \left( x^k - \tfrac{1}{L_h} \nabla h(x^k) \right)\right\|^2.
	\end{equation}
	The iteration complexity of the above process is $\cO\left( \tfrac{L_h}{\mu_h}\log\tfrac{1}{\varepsilon}\right)$. 
	
	Motivated by~\cite{wang2018distributed}\footnote{Iterative process~\eqref{eq:pgd_specialized} is in fact a special case of algorithms proposed in~\cite{wang2018distributed}. See Remark~\ref{rem:graph_transfer} for details.}, there are two different ways to apply the process~\eqref{eq:pgd} to the problem~\eqref{eq:main}. A more straightforward option is to set $h=f, \phi  = \lambda \psi$, which results in the following update rule
	\begin{equation} \label{eq:pgd_specialized_simple}
	x^{k+1}_i =    \tfrac{Ly^k_i + \lambda \bar{y}^k}{L+\lambda}, \quad \text{where} \quad  y^{k}_i =x^{k}_i - \tfrac1L \nabla f(x^k_i),\quad \bar{y}^k = \frac{1}{n}\sum_{i=1}^n y^{k}_i, 
	\end{equation}
	and it yields $\cO\left( \tfrac{L}{\mu}\log\tfrac{1}{\varepsilon}\right)$ rate. The second option is to set $h (x)=  \lambda \psi(x) + \tfrac{\mu}{2n}\| x\|^2$ and $\phi(x) = f(x) -  \tfrac{\mu}{2n}\| x\|^2$. Consequently, the update rule~\eqref{eq:pgd} becomes (see Lemma~\ref{lem:mnadjnjks} in the Appendix):
	\begin{equation}
	x^{k+1}_i = \prox_{\tfrac{1}{\lambda}f_i}(\bar{x}^k) = \argmin_{z\in \R^d} f_i(z) + \tfrac{\lambda}{2} \| z - \bar{x}^k\|^2 \quad \text{for all }i,
	\label{eq:pgd_specialized}
	\end{equation}
	matching the {\tt FedProx}~\cite{li2018federated} algorithm. The iteration complexity we obtain is, however, $\cO\left( \tfrac{\lambda}{\mu}\log\tfrac{1}{\varepsilon}\right)$ (see Lemma~\ref{lem:mnadjnjks} again). 
	
	As both~\eqref{eq:pgd_specialized_simple} and~\eqref{eq:pgd_specialized} require a single communication round per iteration, the corresponding communication complexity becomes $\cO\left(\tfrac{L}{\mu} \log\tfrac1\varepsilon\right)$ and $\cO\left(\tfrac{\lambda}{\mu} \log\tfrac1\varepsilon\right)$ respectively, which is suboptimal in the light of Theorem~\ref{thm:lb}. 
	
	Fortunately, incorporating the Nesterov's momentum~\cite{nesterov1983method, beck2009fast} on top of the procedure~\eqref{eq:pgd_specialized} yields both an optimal communication complexity and optimal local prox complexity once $\lambda\leq L$. We will refer to such method as  {\tt APGD1} (Algorithm~\ref{alg:fista} in the Appendix). Similarly, incorporating the acceleration into~\eqref{eq:pgd_specialized_simple} yields both an optimal communication complexity and optimal local prox complexity once $\lambda\geq L$. Furthermore, such an approach yields the optimal local gradient complexity regardless of the relative comparison of $L, \lambda$. We refer to such method {\tt APGD2} (Algorithm~\ref{alg:fista_2} in the Appendix).

	\subsection{Beyond proximal oracle: Inexact {\tt APGD} ({\tt IAPGD}) \label{sec:iapgd}}
	In most cases, the local proximal oracle is impractical as it requires the exact minimization of the regularized local problem at each iteration. In this section, we describe an accelerated inexact~\cite{schmidt2011convergence} version of~\eqref{eq:pgd_specialized} (Algorithm~\ref{alg:fista_inex}), which only requires a local (either full or summand) gradient oracle. We present two different approaches to achieve so: {\tt AGD}~\cite{nesterov1983method} (under the gradient oracle) and {\tt Katyusha}~\cite{allen2017katyusha} (under the summand gradient oracle). Both strategies, however, share a common characteristic: they progressively increase the effort to inexactly evaluate the local prox, which is essential in order to preserve the optimal communication complexity.

	\begin{algorithm}[h]
		\caption{{\tt IAPGD} +$ \cA$}
		\label{alg:fista_inex}
		\begin{algorithmic}
			\REQUIRE Starting point $y^0 = x^0\in\R^{nd}$
			\FOR{ $k=0,1,2,\ldots$ }
			\STATE{ {\color{blue}Central server} computes the average $\bar{y}^k = \frac1n \sum_{i=1}^n y^k_i$}
			\STATE{For all {\color{red} clients} $i=1,\dots,n$: 
				\STATE \quad Set $h_i^{k+1}(z) \eqdef  f_i(z) + \frac{\lambda}{2} \| z -\bar{y}^k \|^2$ and find $x^{k+1}_i$ using local solver $\cA$ for $T_k$ iterations
								   \vskip-0.2cm
				\begin{equation} \label{eq:algo_suboptimality}
				h_{i}^{k+1}(x^{k+1}_i) \leq \epsilon_k+ \min_{z\in \R^d}  h_{i}^{k+1}(z).
				\end{equation}
			}
			\STATE{         \vskip-0.35cm
				For all  {\color{red} clients} $i=1,\dots,n$: Take the momentum step $y^{k+1}_i = x^{k+1}_i +\frac{ \sqrt{\lambda}- \sqrt{\mu}}{ \sqrt{\lambda}+ \sqrt{\mu}} ( x^{k+1}_i - x^{k}_i ) $}
			\ENDFOR
		\end{algorithmic}
	\end{algorithm}
	
	\begin{remark}\label{rem:graph_transfer}
		As already mentioned, the idea of applying {\tt IAPGD} to solve~\eqref{eq:main} is not new; it was already explored in~\cite{wang2018distributed}.\footnote{The work \cite{wang2018distributed} considers the distributed multi-task learning objective that is more general than~\eqref{eq:main}.} However,~\cite{wang2018distributed} does not argue about the optimality of {\tt IAPGD}. Less importantly, our analysis is slightly more careful, and it supports {\tt Katyusha} as a local sub-solver as well.  
	\end{remark}

	{\bf {\tt IAPGD} + {\tt AGD}}
	
	The next theorem states the convergence rate of {\tt IAPGD} with {\tt AGD}~\cite{nesterov1983method} as a local subsolver.

	\begin{theorem}\label{thm:inexact}
		Suppose that $f_i$ is $L$-smooth and $\mu$-strongly convex for all $i$. Let {\tt AGD}  with starting point $y^k_i$ be employed for
		$T_k \eqdef   \sqrt{\tfrac{L+ \lambda}{\mu+\lambda}} \log \left( 1152 L \lambda n^2 \left(2 \sqrt{\tfrac{\lambda}{\mu}} +1 \right)^2\mu^{-2}\right)
		+4\sqrt{\tfrac{\mu(L+ \lambda)}{\lambda(\mu+\lambda)}} k$
		iterations to approximately solve~\eqref{eq:algo_suboptimality} at iteration $k$.
		Then, we have
		$
		F(x^k) -  F^\star
		\leq
		8 \left( 1- \sqrt{\tfrac{\mu}{\lambda}}\right)^k (F(x^0) - F^\star),$ 
		where $F^\star = F(x^\star)$.
		As a result, the total number of communications required to reach $\varepsilon$-approximate solution is $\cO\left( \sqrt{\tfrac{\lambda}{\mu}}\log\tfrac1\varepsilon\right)$. The corresponding local gradient complexity is
		$
		\cO\left(
		\sqrt{\tfrac{L+ \lambda}{\mu}} \log\tfrac1\varepsilon
		\left( \log \tfrac{ L \lambda n}{\mu} +\log\tfrac1\varepsilon \right)
		\right) = \tilde{\cO}\left(\sqrt{\tfrac{L+ \lambda}{\mu}}  \right)$.
	\end{theorem}
	
	As expected, the communication complexity of {\tt IAPGD} + {\tt AGD} is $\cO\left( \sqrt{\tfrac\lambda\mu} \log \tfrac1\varepsilon\right)$, thus optimal. On the other hand, the local gradient complexity is $\tilde{\cO}\left(\sqrt{\tfrac{L+ \lambda}{\mu}}  \right)$. For $\lambda = \cO(L)$ this simplifies to $\tilde{\cO}\left(\sqrt{\tfrac{L}{\mu}}  \right)$, which is, up to $\log$ and constant factors identical to the lower bound on the local gradient calls.

{\bf {\tt IAPGD} + {\tt Katyusha}}

	In practice, the local objectives $f_i$'s often correspond to a loss of some model on the given client's data. In such a case, each function $f_i$ is of the finite-sum structure:
	$
	f_\tR(x_{\tR}) = \tfrac1m \sum_{j =1}^m \flocc_{i,j}(x_{\tR}).
	$

	Clearly, if $m$ is large, solving the local subproblem with {\tt AGD} is rather inefficient as it does not take an advantage of the finite-sum structure. To tackle this issue, we propose solving the local subproblem~\eqref{eq:algo_suboptimality} using {\tt Katyusha}.\footnote{Essentially any accelerated variance reduced algorithm can be used instead of {\tt Katyusha}, for example {\tt ASDCA}~\cite{shalev2014accelerated}, {\tt APCG}~\cite{lin2015accelerated}, Point-{\tt SAGA}~\cite{defazio2016simple}, {\tt MiG}~\cite{zhou2018simple}, {\tt SAGA-SSNM}~\cite{zhou2018direct} and others.}

	\begin{theorem}\label{thm:inexact_stoch}
		Let $\flocc_{i,j}$ be $\Lloc$-smooth and $f_i$ be $\mu$-strongly convex for all $1\leq i\leq n, 1\leq j\leq m$.\footnote{Consequently, we have $\Lloc\geq L \geq \tfrac{\Lloc}{m}$.}
		
		Let {\tt Katyusha}  with starting point $y^k_i$ be employed for
		$T_k = 
		\cO\left(\left(m+  \sqrt{m \tfrac{\Lloc + \lambda}{\mu + \lambda}}\right)\left( \log\tfrac{1}{R^2} +  k\sqrt{\tfrac{\mu}{\lambda}} \right)\right)$
		iterations to approximately solve~\eqref{eq:algo_suboptimality} at iteration $k$ of {\tt IAPGD} for some small $R$ (see proof for details). Given that the iterate sequence $\{x^k\}_{k=0}^\infty$ is bounded, the expected communication complexity of {\tt IAPGD}+{\tt Katyusha} is $\cO\left(\sqrt{\frac{\lambda}{\mu}} \log \frac1\epsilon\right)$, while the local summand gradient complexity is
		$\tilde{\cO}\left(m\sqrt{\tfrac{\lambda}{\mu}} + \sqrt{m\tfrac{\Lloc}{\mu}} \right) $.
	\end{theorem}
	
	Theorem~\ref{thm:inexact_stoch} shows that local {\tt Katyusha} enjoys the optimal communication complexity. Furthermore, if $\sqrt{m}\lambda = \cO(\Lloc)$, the total expected number of local gradients becomes optimal as well (see Sec.~\ref{sec:lower_local}).

	There is, however, a notable drawback of Theorem~\ref{thm:inexact_stoch} over Theorem~\ref{thm:inexact} -- Theorem~\ref{thm:inexact_stoch} requires a boundedness of the sequence $\{x^k\}_{k=0}^\infty$ as an assumption, while this piece is not required for {\tt IAPGD}+{\tt AGD} due to its deterministic nature. In the next section, we devise a stochastic algorithm {\tt AL2SGD+} that does not require such an assumption. Furthermore, the local (summand) gradient complexity of {\tt AL2SGD+} does not depend on the extra log factors, and, at the same time, {\tt AL2SGD+} is optimal in a broader range of scenarios.

	\subsection{Accelerated {\tt L2SGD+}}
	\label{sec:al2sgd+}

	In this section, we introduce accelerated version of {\tt L2SGD+}~\cite{hanzely2020federated}, which can be viewed as a variance-reduced variant of {\tt FedAvg} devised to solve~\eqref{eq:main}. The proposed algorithm, {\tt AL2SGD+}, is stated as Algorithm~\ref{alg:acc_stoch} in the Appendix. From high-level point of view, {\tt AL2SGD+} is nothing but {\tt L-Katyusha} with non-uniform minibatch sampling.\footnote{ {\tt L-Katyusha}~\cite{qian2019svrg} is a variant of {\tt Katyusha} with a random inner loop length.} In contrast to the approach from Section~\ref{sec:iapgd}, {\tt AL2SGD+} does not treat $f$ as a proximable regularizer, but rather directly constructs $g^k$--a non-uniform minibatch variance reduced stochastic estimator of $\nabla F(x^k)$. Next, we state the communication and the local summand gradient complexity of {\tt AL2SGD+}.
	
	\begin{theorem} \label{thm:a2}
		Suppose that the parameters of {\tt AL2SGD+} are chosen as stated in Proposition~\ref{prop:acc} in the Appendix.  In such case, the communication complexity of {\tt AL2SGD+} with $\probx = \proby(1-\proby)$, where $\proby = \frac{\lambda}{\lambda+\Lloc}$, is 
		$
		\cO\left(  \sqrt{\tfrac{\min\{\Lloc, \lambda \}}{\mu}}\log\tfrac1\varepsilon\right)
		$
		(see Sec.~\ref{sec:a2_proof} of the Appendix)
		while the local gradient complexity of {\tt AL2SGD+} for $\probx = \frac1m$ and $\proby = \frac{\lambda}{\lambda+\Lloc}$ is
		$
		\cO\left(
		\left(m +
		\sqrt{\tfrac{m (\Lloc +  \lambda)}{\mu}}\right)\log\tfrac1\varepsilon
		\right).
		$
		
	\end{theorem}
	
	The communication complexity of {\tt AL2SGD+} is optimal regardless of the relative comparison of $\Lloc, \lambda$, which is an improvement over the previous methods. Furthermore, {\tt AL2SGD+} with a slightly different parameters choice enjoys the local gradient complexity which is optimal once $\lambda = \cO(\Lloc)$.

	\section{Experiments}
	
 In this section we present empirical evidence to support the theoretical claims of this work. 
	
	In the first experiment, we study the most practical scenario with where the local objective is of a finite-sum structure, while the local oracle provides us with gradients of the summands. In this work, we developed two algorithms capable of dealing with the summand oracle efficiently: {\tt IAPGD}+{\tt Katyusha} and {\tt AL2SGD+}. We compare both methods against the baseline {\tt L2SGD+} from~\cite{hanzely2020federated}.The results are presented in Figure \ref{fig:com_met}. In terms of the number of communication rounds, both {\tt AL2SGD+} and {\tt IAPGD+Katyusha} are significantly superior to the {\tt L2SGD+}, as theory predicts. The situation is, however, very different when looking at the local computation. While {\tt AL2SGD+} performs clearly the best, {\tt IAPGD+Katyusha} falls behind {\tt L2SGD}. We presume this happened due to the large constant and $\log$ factors in the local complexity of {\tt IAPGD+Katyusha}.

	\begin{figure}[!h]
		\centering
		\begin{minipage}{0.3\textwidth}
			\centering
			\includegraphics[width =  \textwidth ]{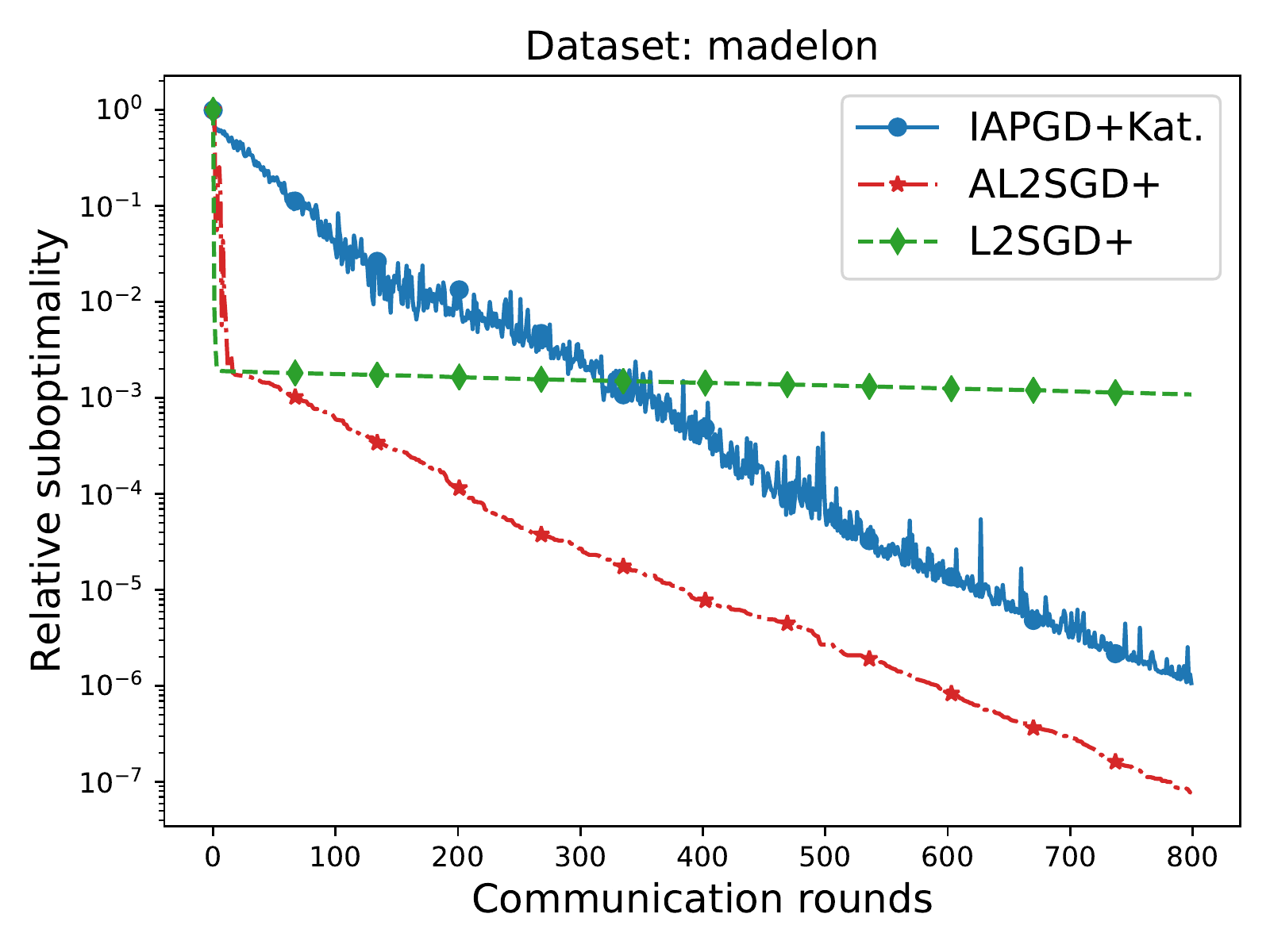}
		\end{minipage}%
		\begin{minipage}{0.3\textwidth}
			\centering
			\includegraphics[width =  \textwidth ]{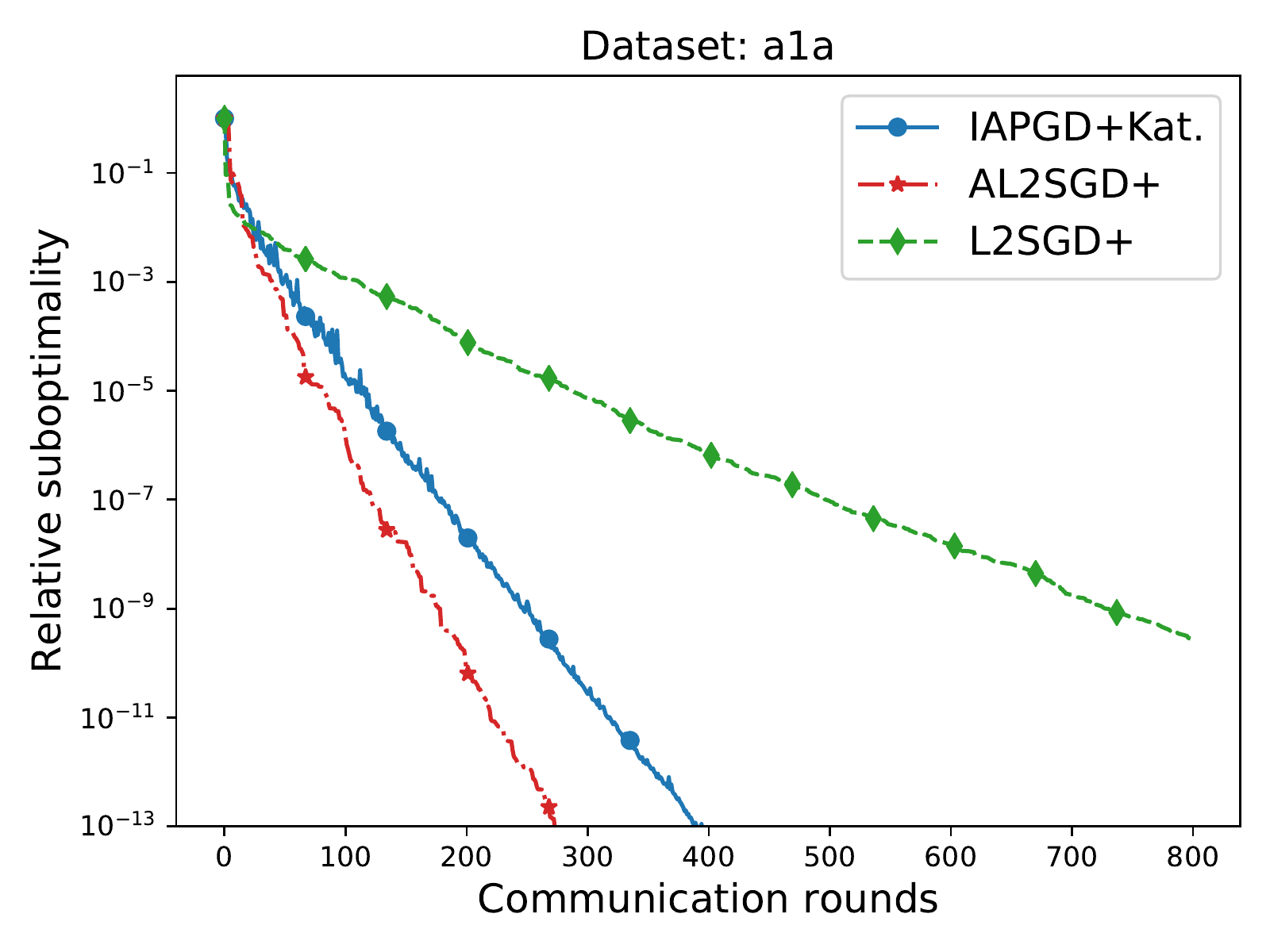}
		\end{minipage}
		\begin{minipage}{0.3\textwidth}
			\centering
			\includegraphics[width =  \textwidth ]{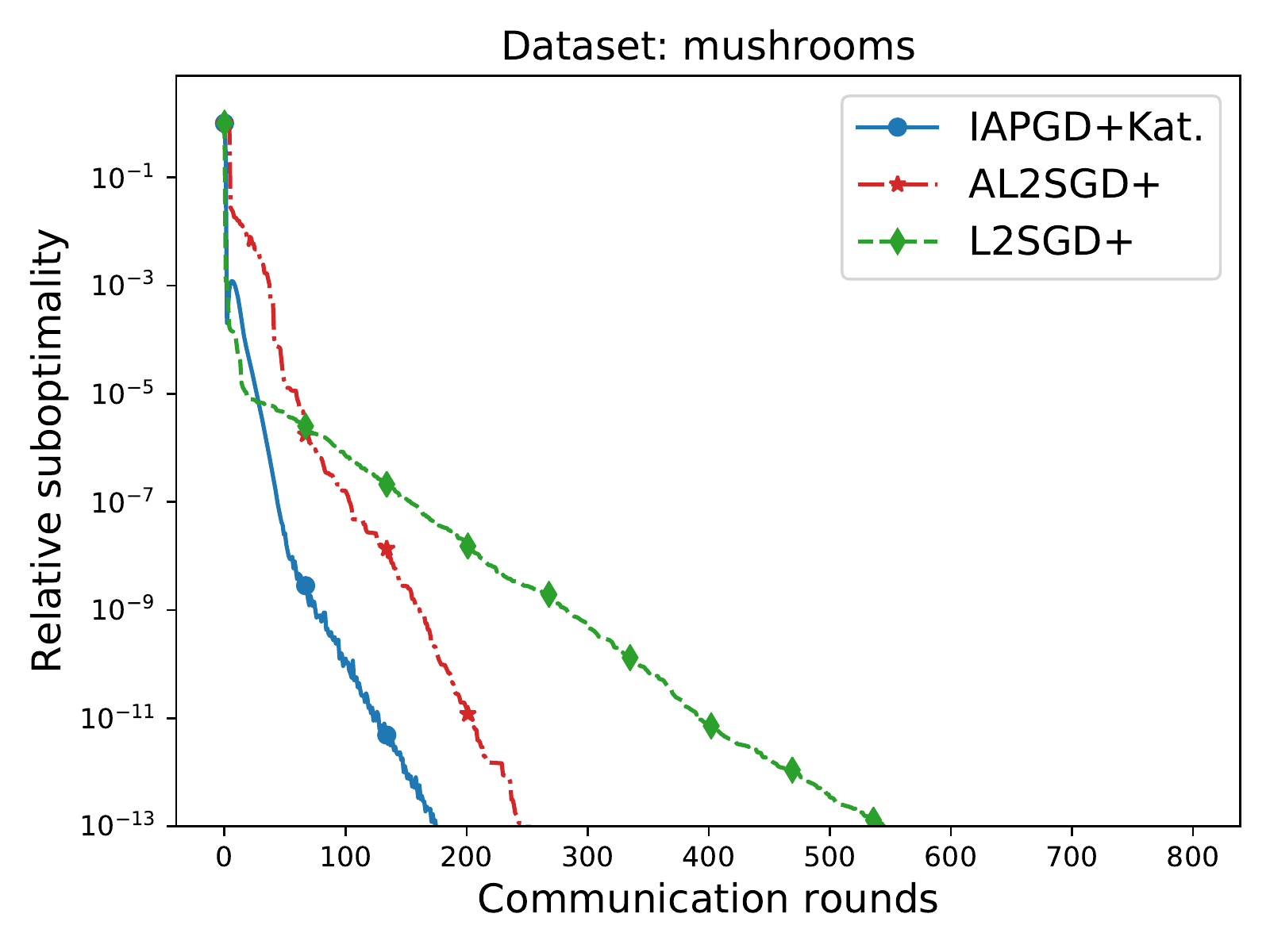}
		\end{minipage}%
		\\
		\begin{minipage}{0.3\textwidth}
			\centering
			\includegraphics[width =  \textwidth ]{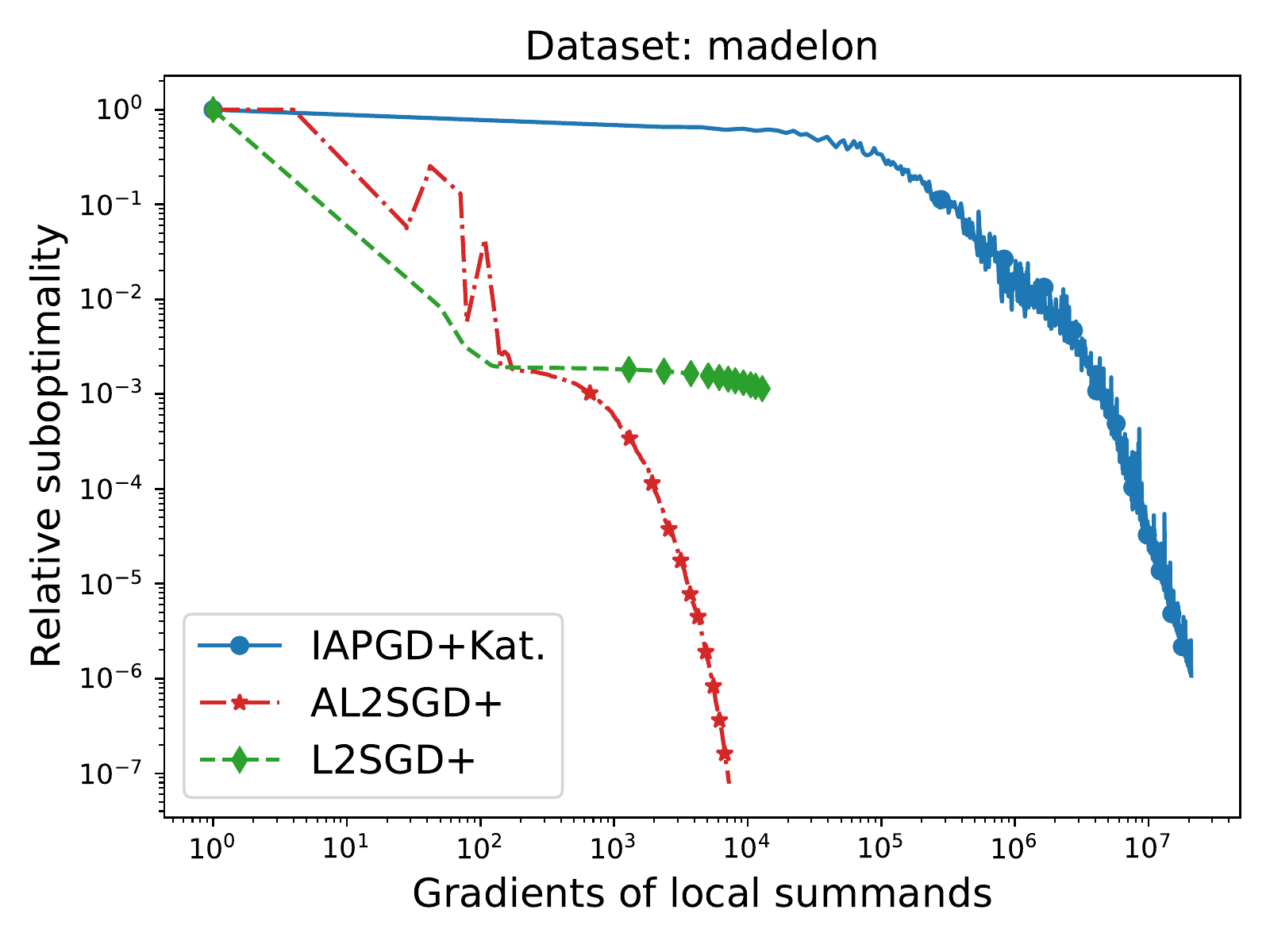}
		\end{minipage}%
		\begin{minipage}{0.3\textwidth}
			\centering
			\includegraphics[width =  \textwidth ]{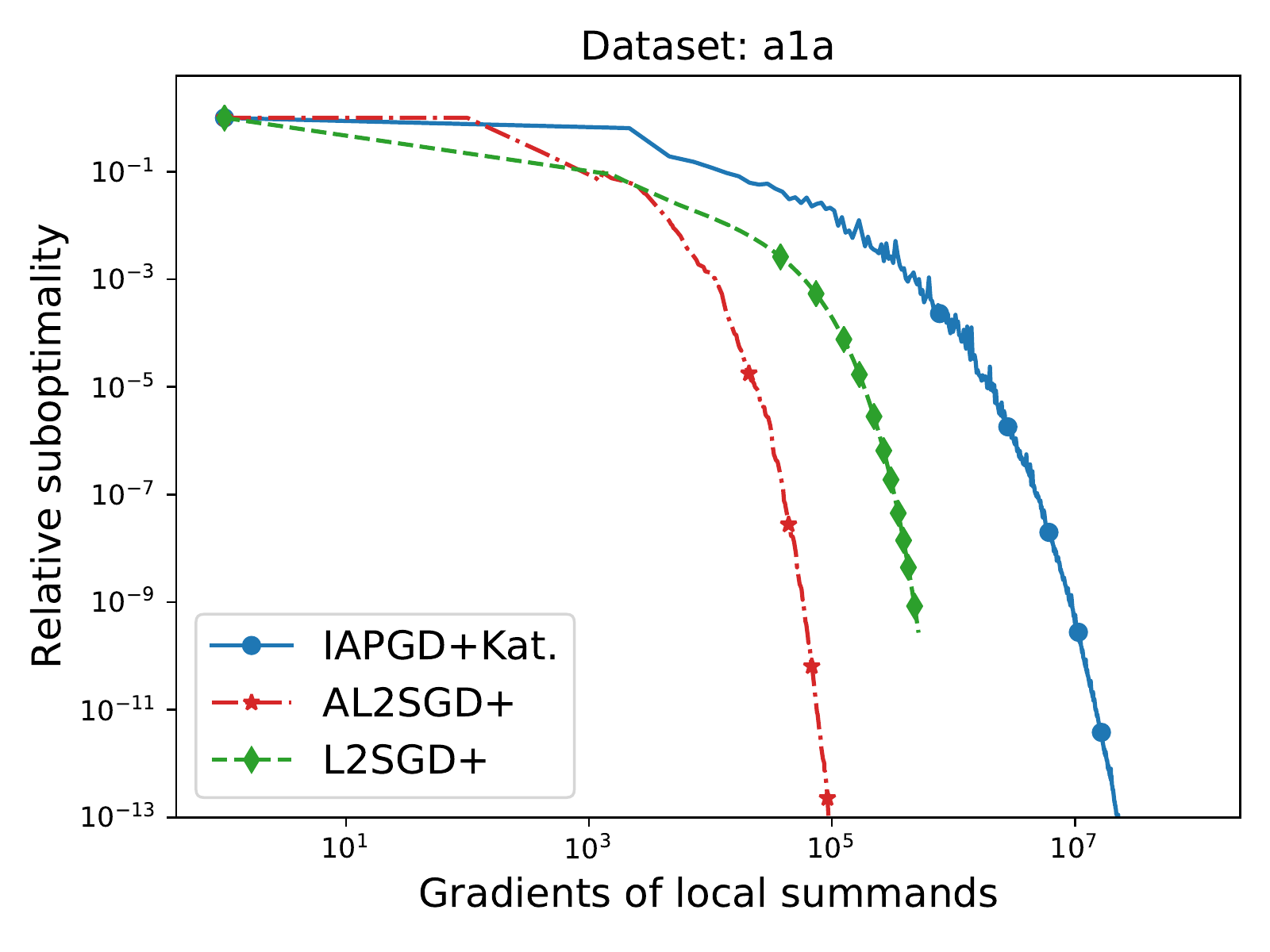}
		\end{minipage}%
		\begin{minipage}{0.3\textwidth}
			\centering
			\includegraphics[width =  \textwidth ]{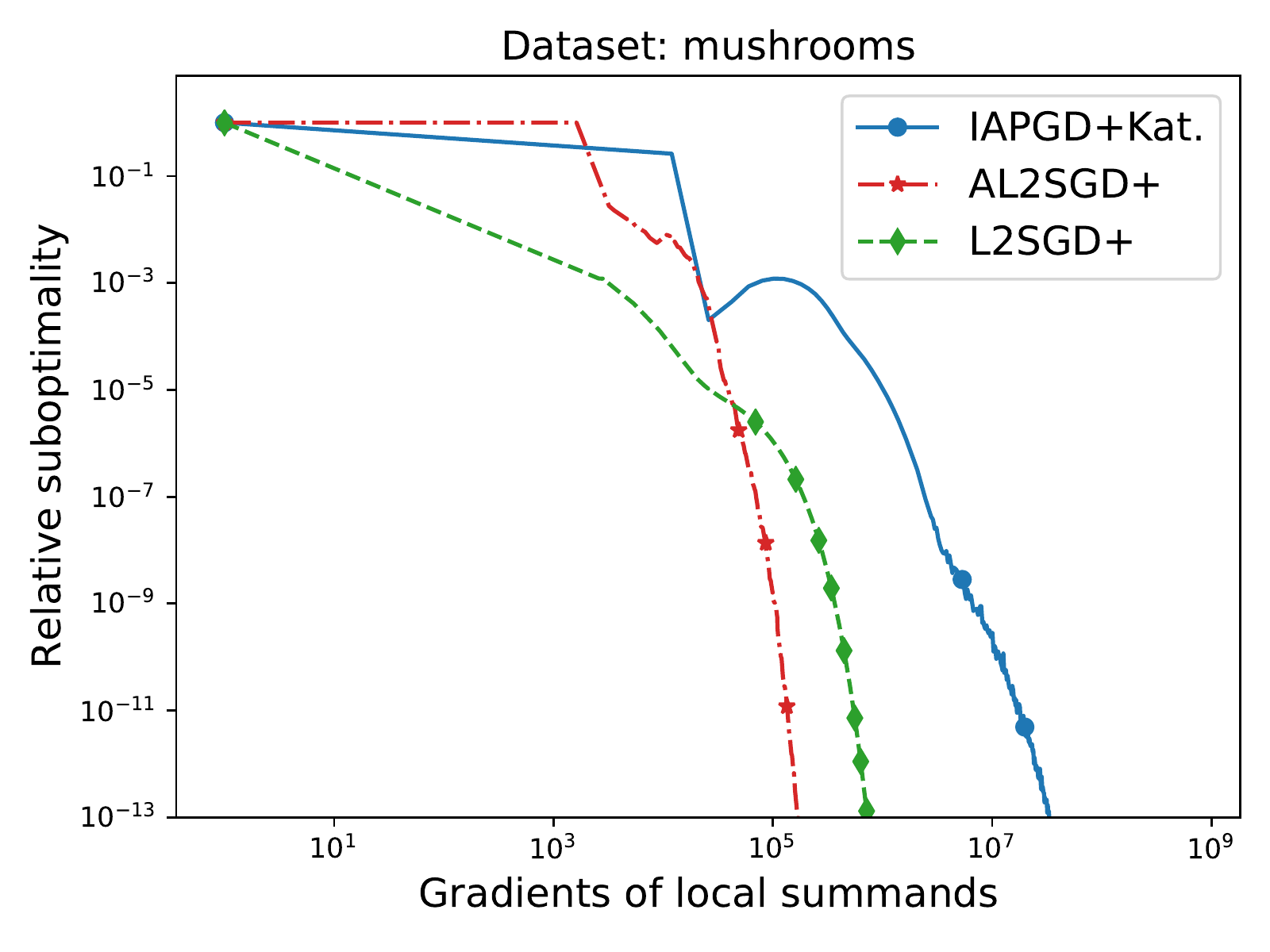}
		\end{minipage}%
		\\
		\begin{minipage}{0.3\textwidth}
			\centering
			\includegraphics[width =  \textwidth ]{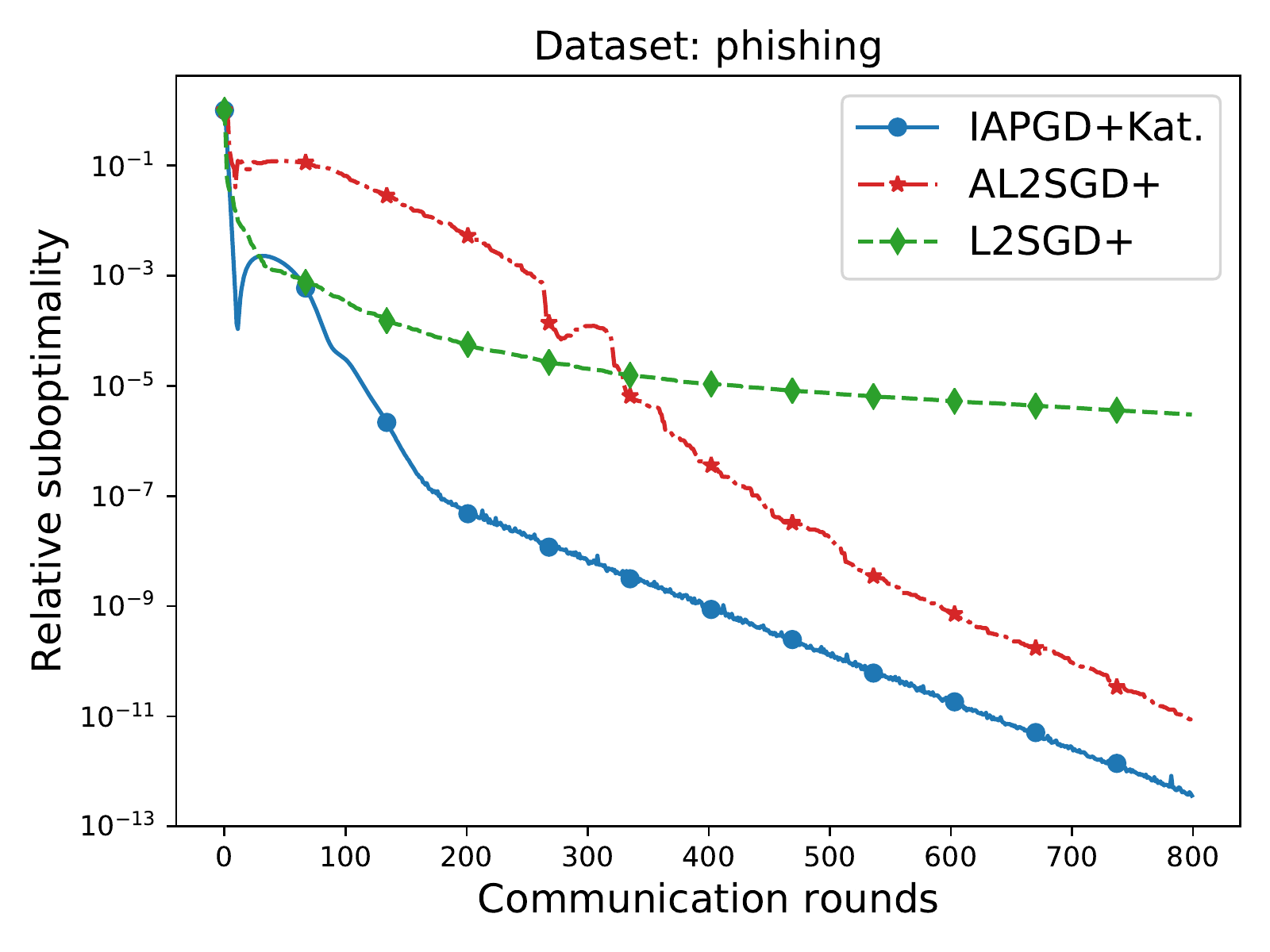}
		\end{minipage}%
		\begin{minipage}{0.3\textwidth}
			\centering
			\includegraphics[width =  \textwidth ]{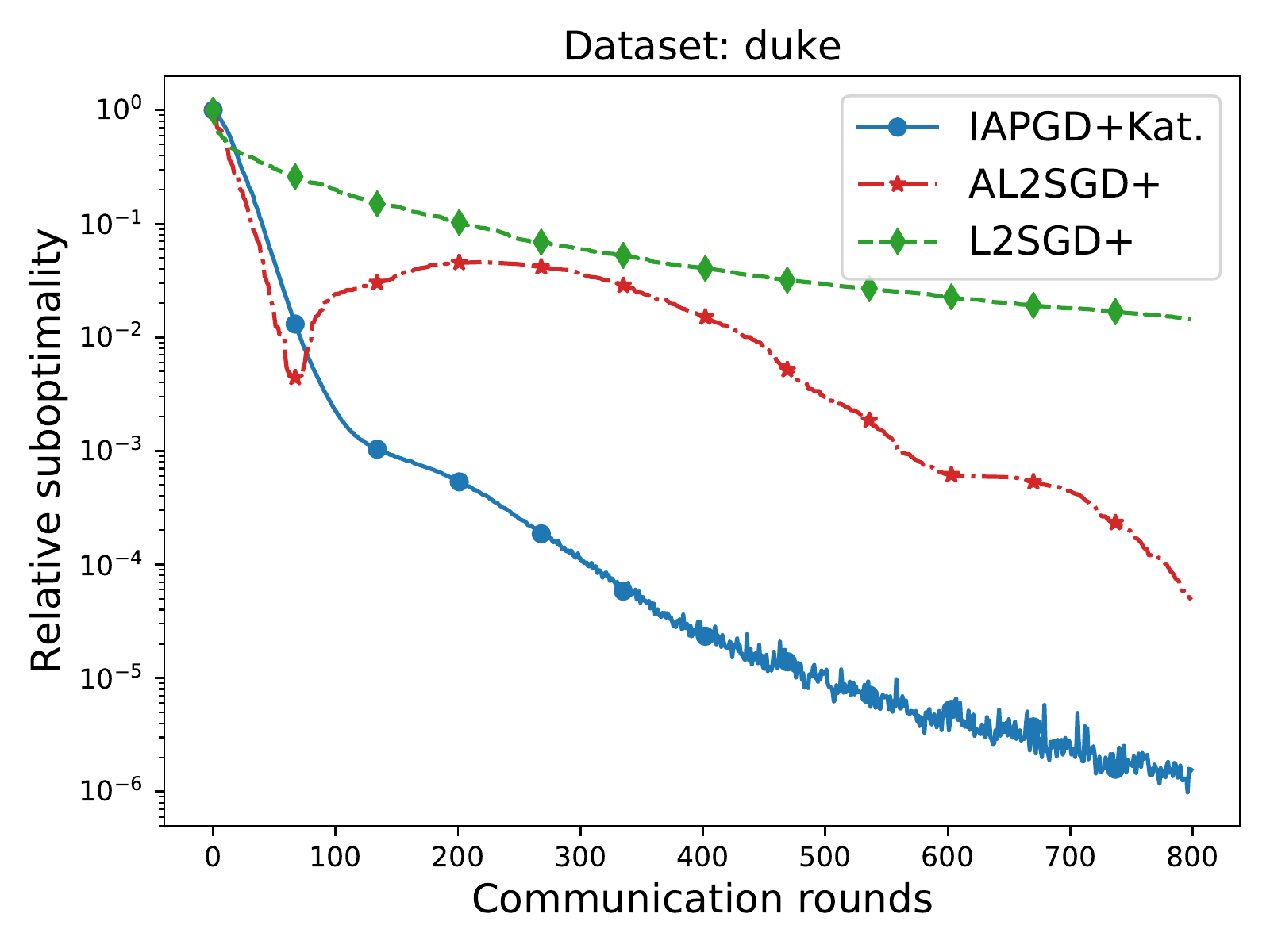}
		\end{minipage}%
		\\
				\begin{minipage}{0.3\textwidth}
			\centering
			\includegraphics[width =  \textwidth ]{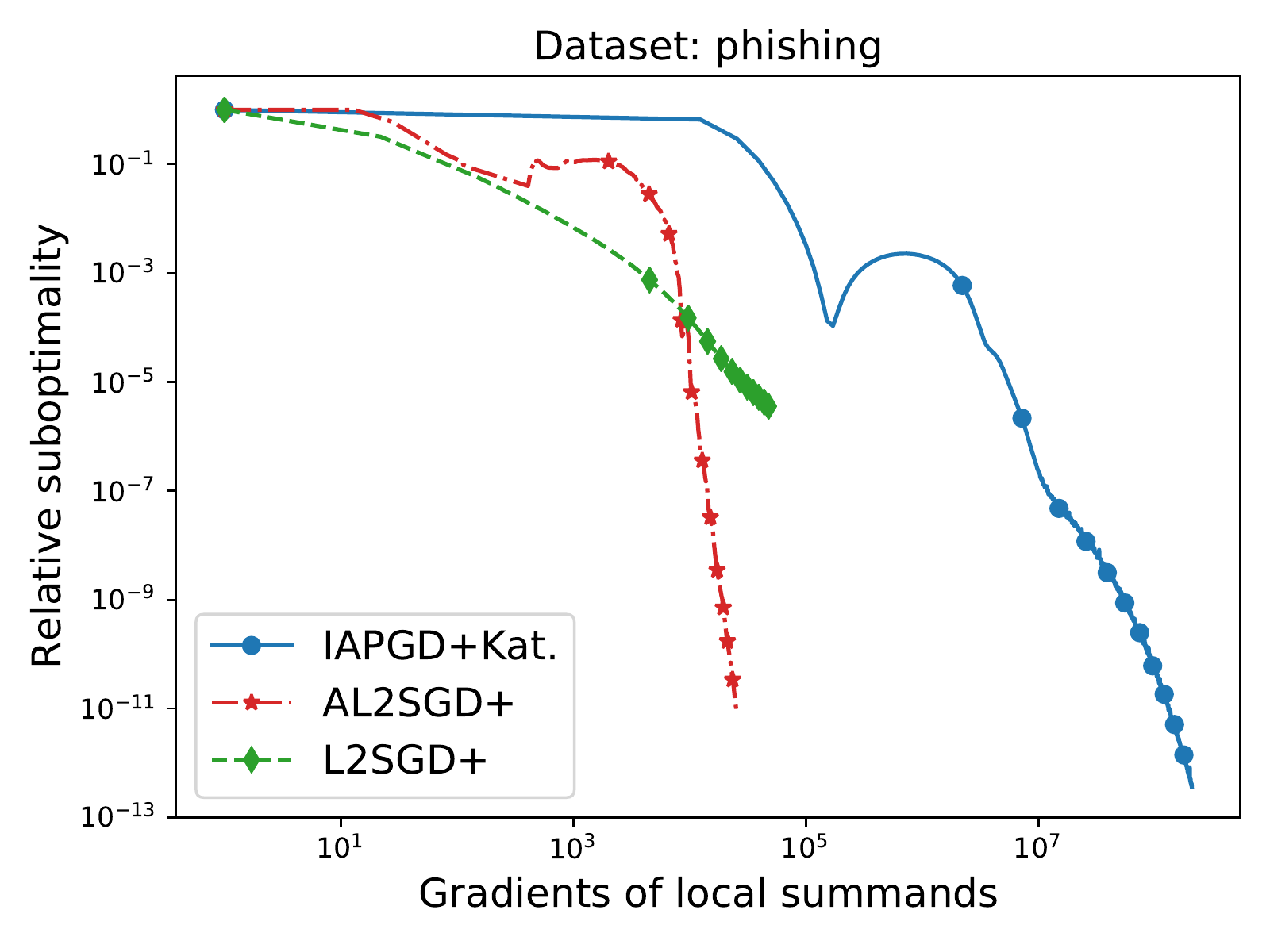}
		\end{minipage}%
				\begin{minipage}{0.3\textwidth}
			\centering
			\includegraphics[width =  \textwidth ]{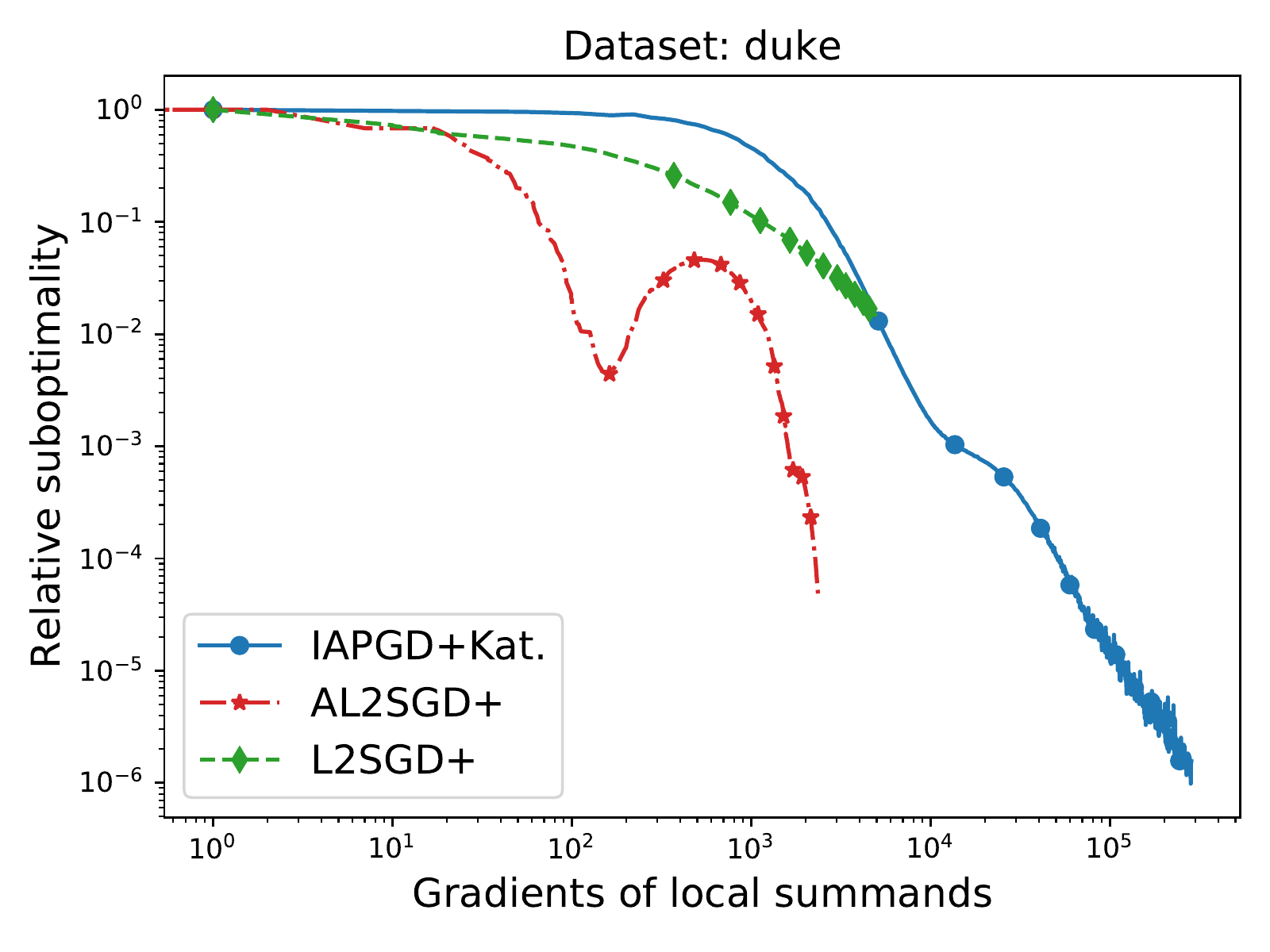}
		\end{minipage}%
		\caption{Comparison of {\tt IAPGD+Katyusha}, {\tt AL2SGD+} and {\tt L2SDG+} on logistic regression with \texttt{LIBSVM} datasets~\cite{chang2011libsvm}. Each client owns a random, mutually disjoint subset of the full dataset. First row: communication complexity, second row: local computation complexity for the same experiment.
		}
		\label{fig:com_met}
	\end{figure}

In the second experiment, we investigate the heterogeneous split of the data among the clients for the same setup as described in the previous paragraph. Figure~\ref{fig:stoch_hetero} shows the result. We can see that the data heterogeneity does not influence the convergence significantly and we observe a similar behaviour compared to the homogenous case.

	\begin{figure}[!h]
		\centering
		\begin{minipage}{0.3\textwidth}
			\centering
			\includegraphics[width =  \textwidth ]{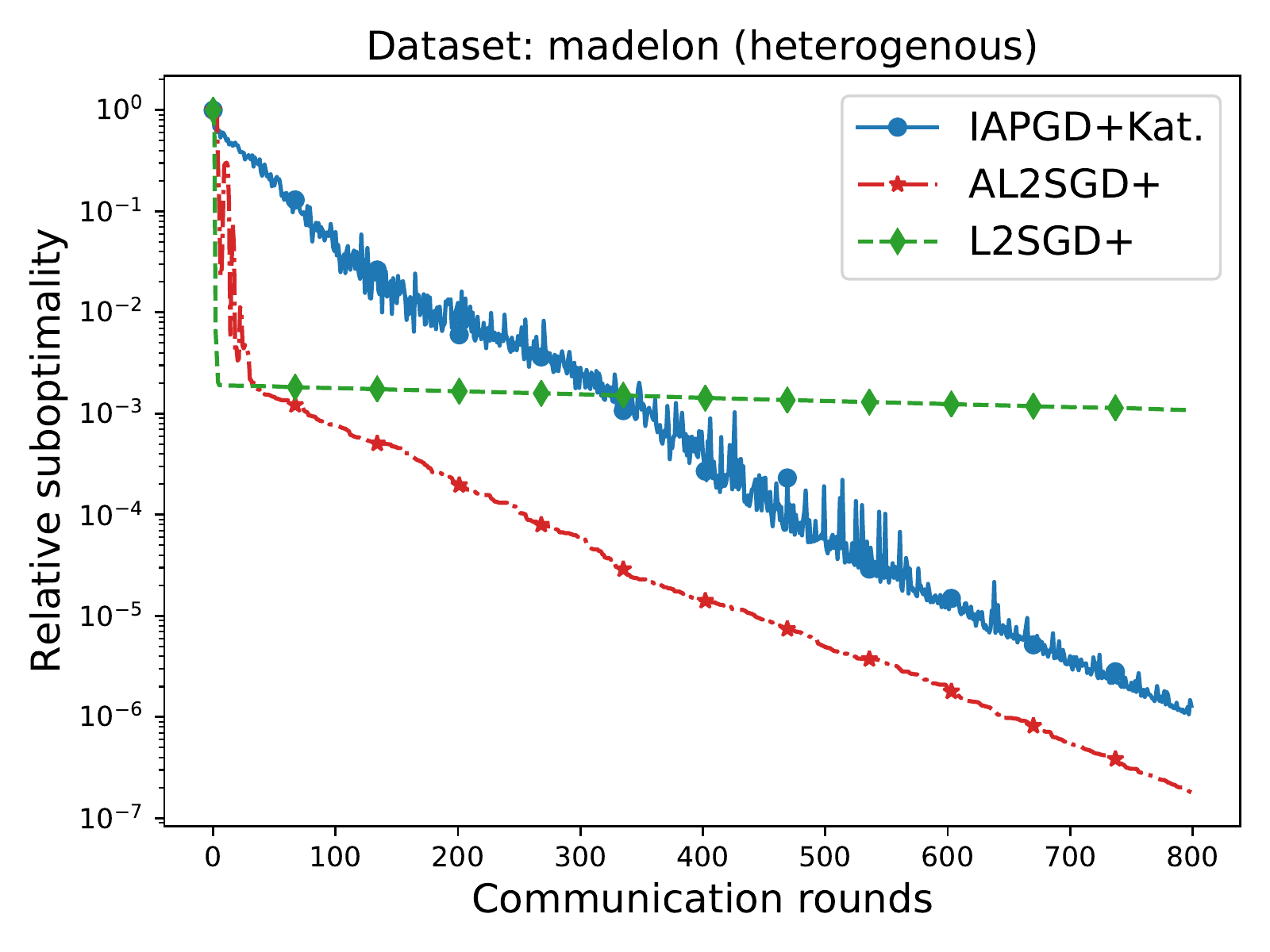}
		\end{minipage}%
		\begin{minipage}{0.3\textwidth}
			\centering
			\includegraphics[width =  \textwidth ]{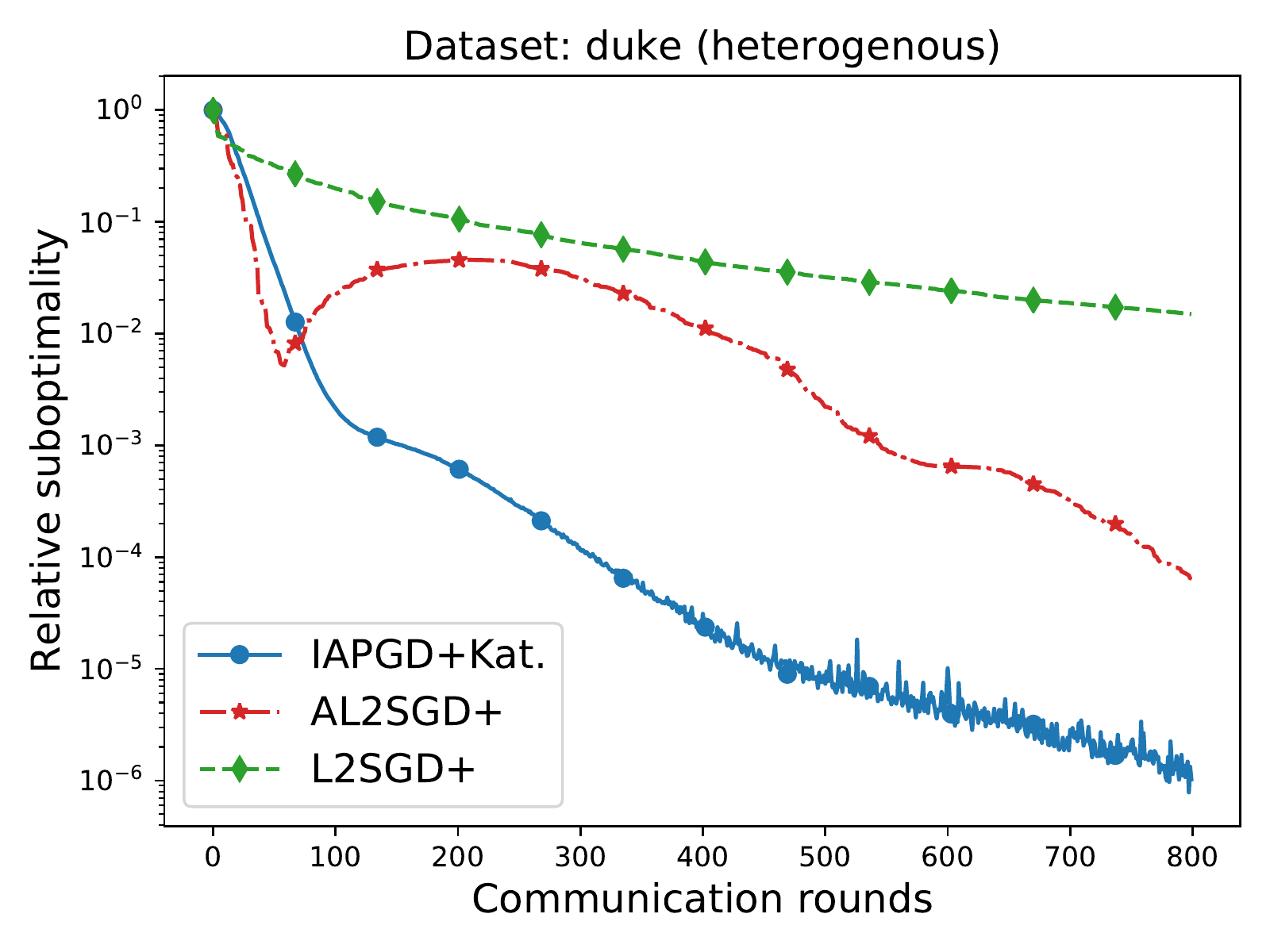}
		\end{minipage}
		\begin{minipage}{0.3\textwidth}
			\centering
			\includegraphics[width =  \textwidth ]{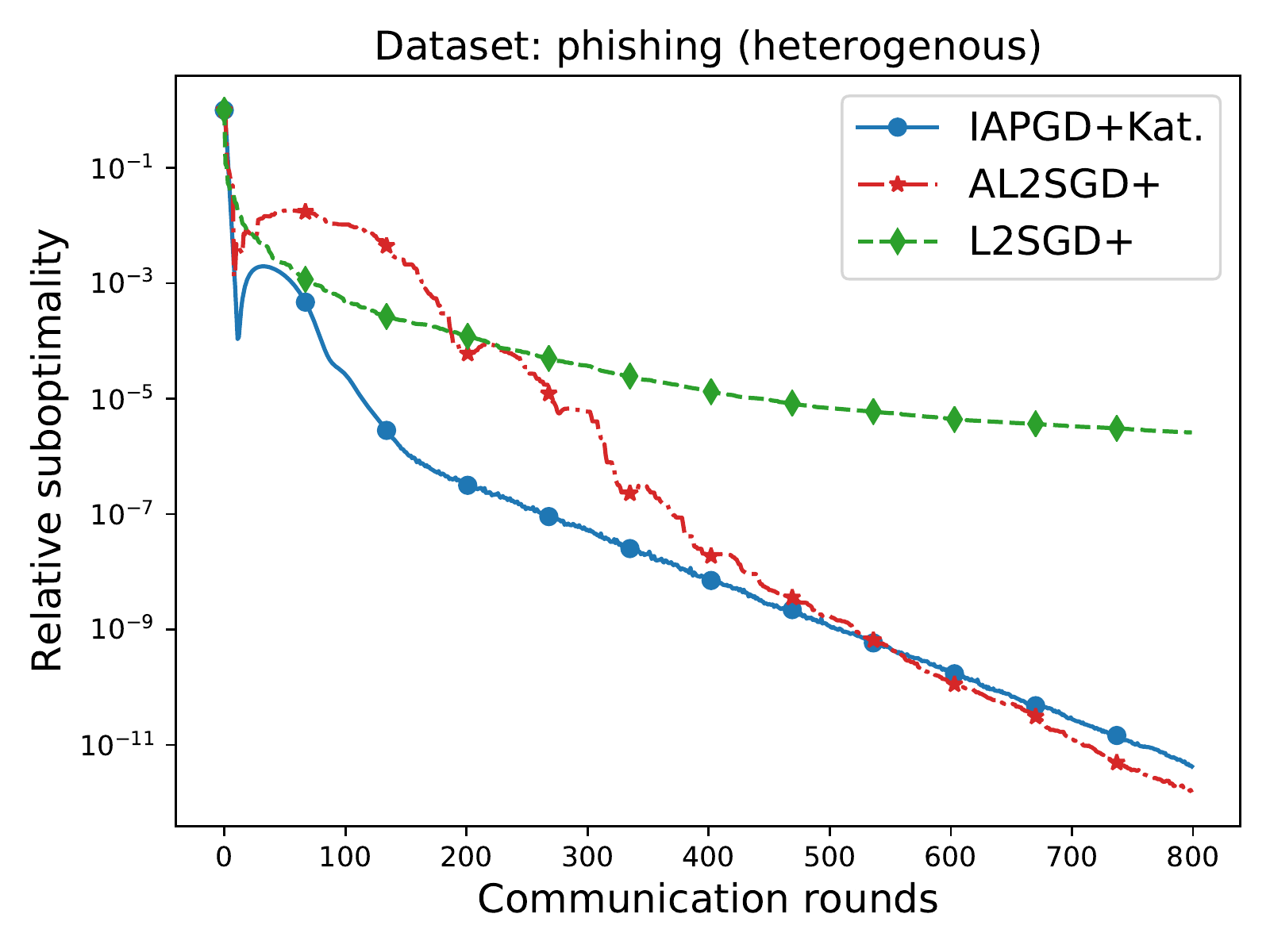}
		\end{minipage}%
		\\
		\begin{minipage}{0.3\textwidth}
			\centering
			\includegraphics[width =  \textwidth ]{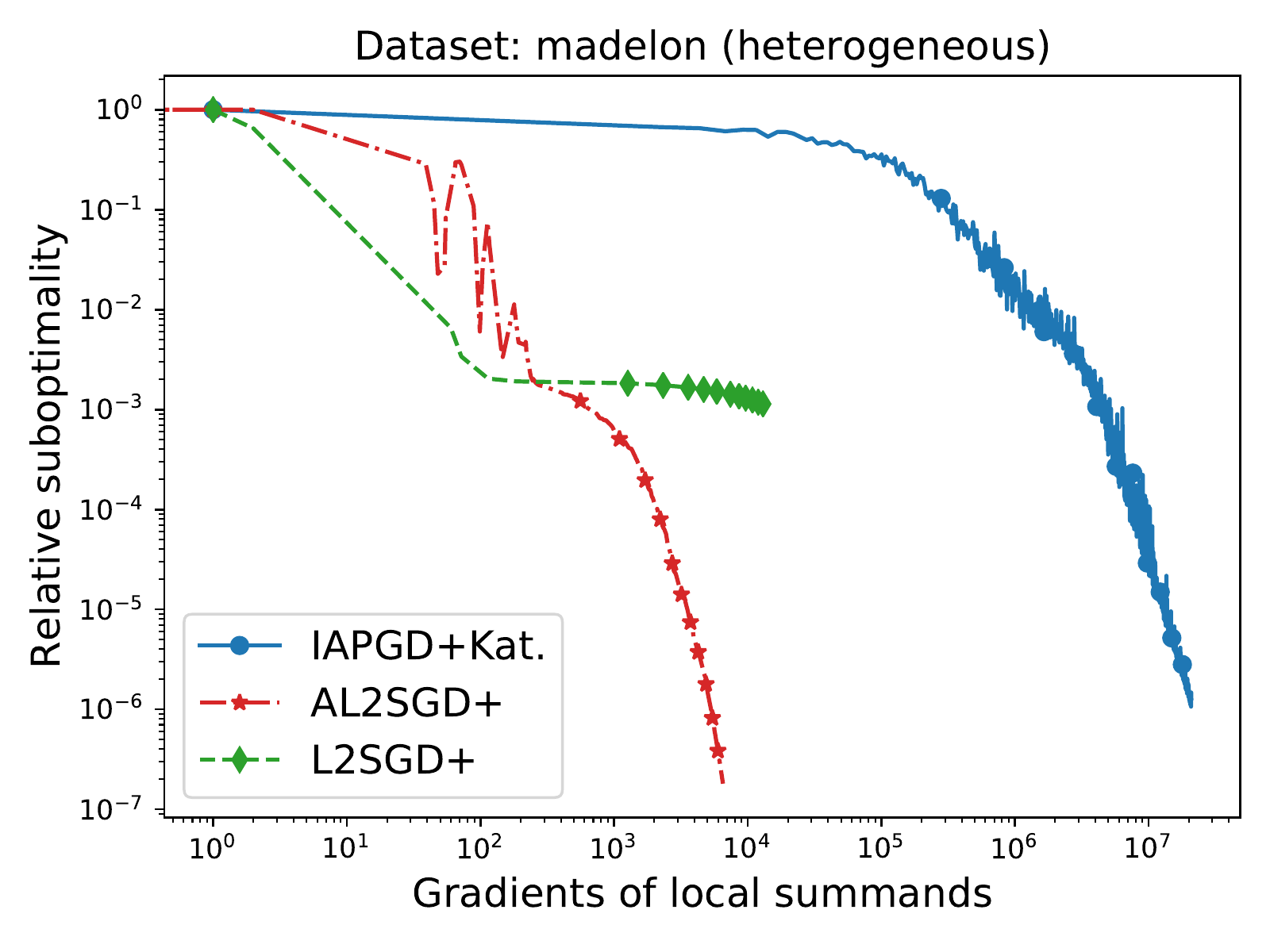}
		\end{minipage}%
		\begin{minipage}{0.3\textwidth}
			\centering
			\includegraphics[width =  \textwidth ]{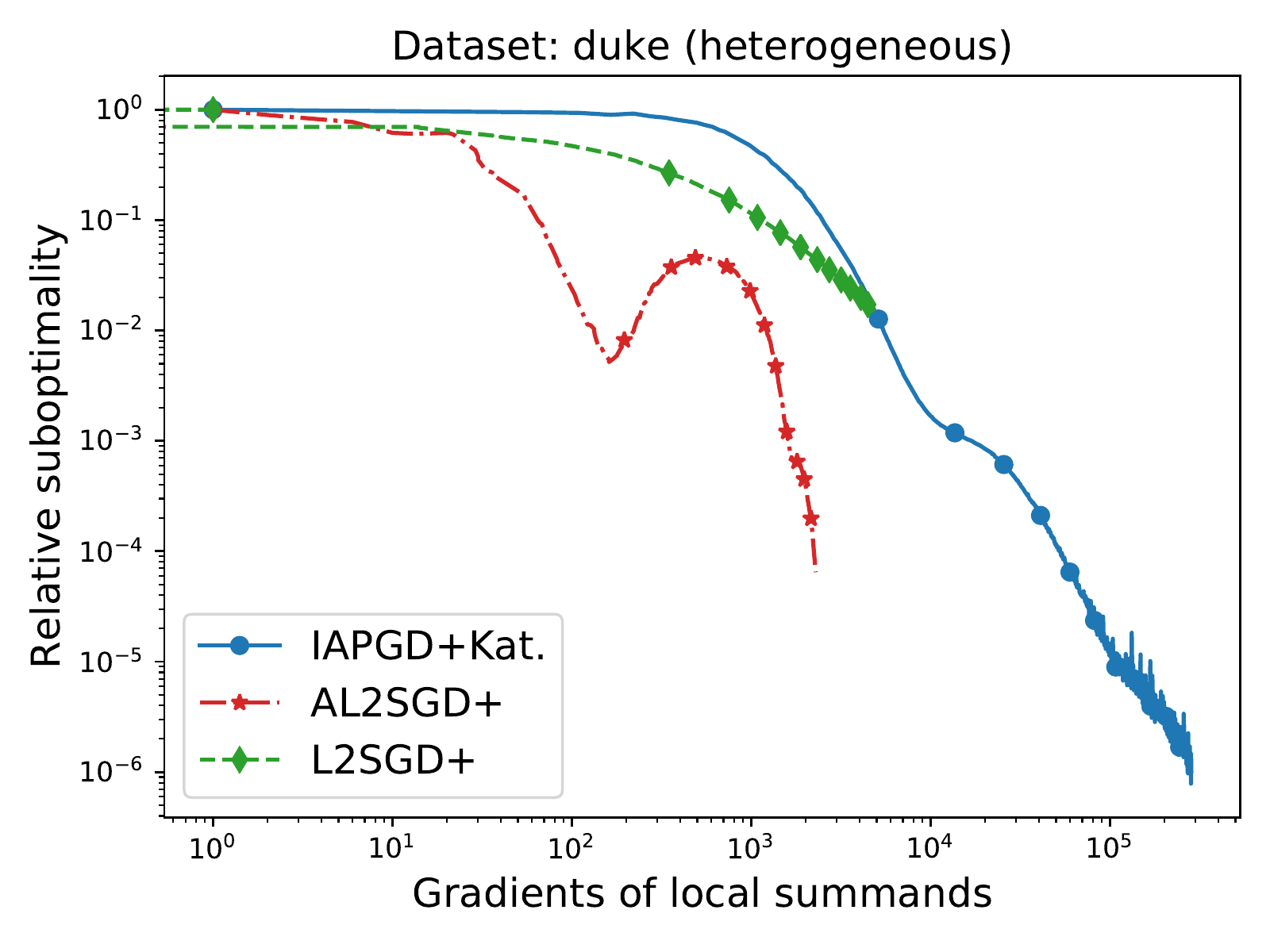}
		\end{minipage}%
		\begin{minipage}{0.3\textwidth}
			\centering
			\includegraphics[width =  \textwidth ]{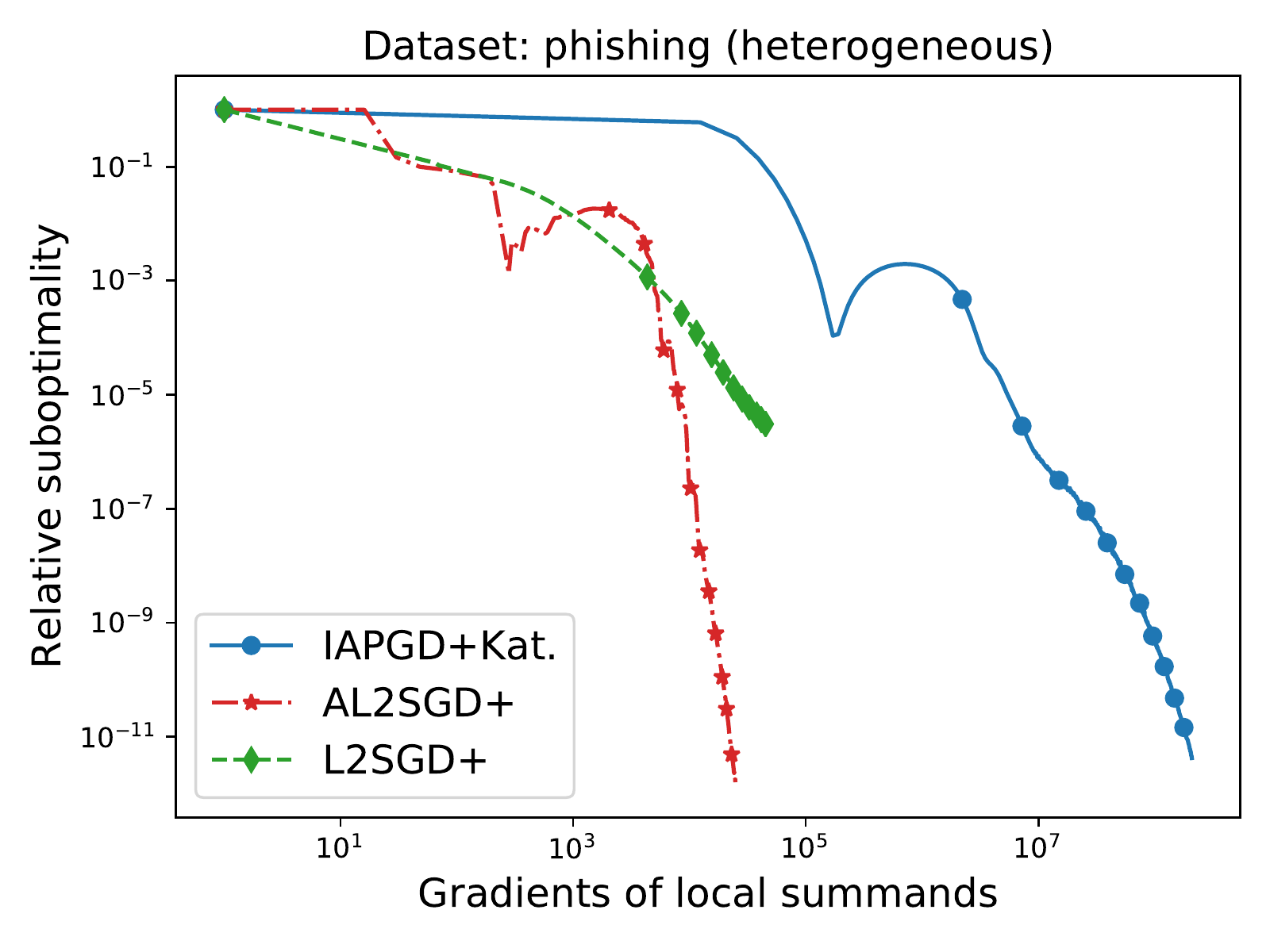}
		\end{minipage}%
\caption{Same experiment as Figure~\ref{fig:com_met}, but a heterogeneous data split. }
			\label{fig:stoch_hetero}
	\end{figure}

	In the third experiments, we compare two variants of {\tt APGD} presented in Section~\ref{sec:apgd_simple}: {\tt APGD1} (Algorithm~\ref{alg:fista}) and {\tt APGD2} (Algorithm~\ref{alg:fista_2}). We consider several synthetic instances of~\eqref{eq:main} where we vary $\lambda$ and keep remaining parameters (i.e., $L, \mu$) fixed. Our theory predicts that while the rate of {\tt APGD2} should not be influenced by varying $\lambda$, the rate of {\tt APGD1} should grow as $\cal{O}(\sqrt{\lambda})$. Similarly, {\tt APGD1} should be favourable if $\lambda\leq L=1$, while {\tt APGD2} should be the algorithm of choice for $\lambda > L=1$. As expected, Figure~\ref{fig:fistas} confirms both claims. 

	\begin{figure}[!h]
		\begin{center}
			\centerline{
				\includegraphics[width=0.4\columnwidth]{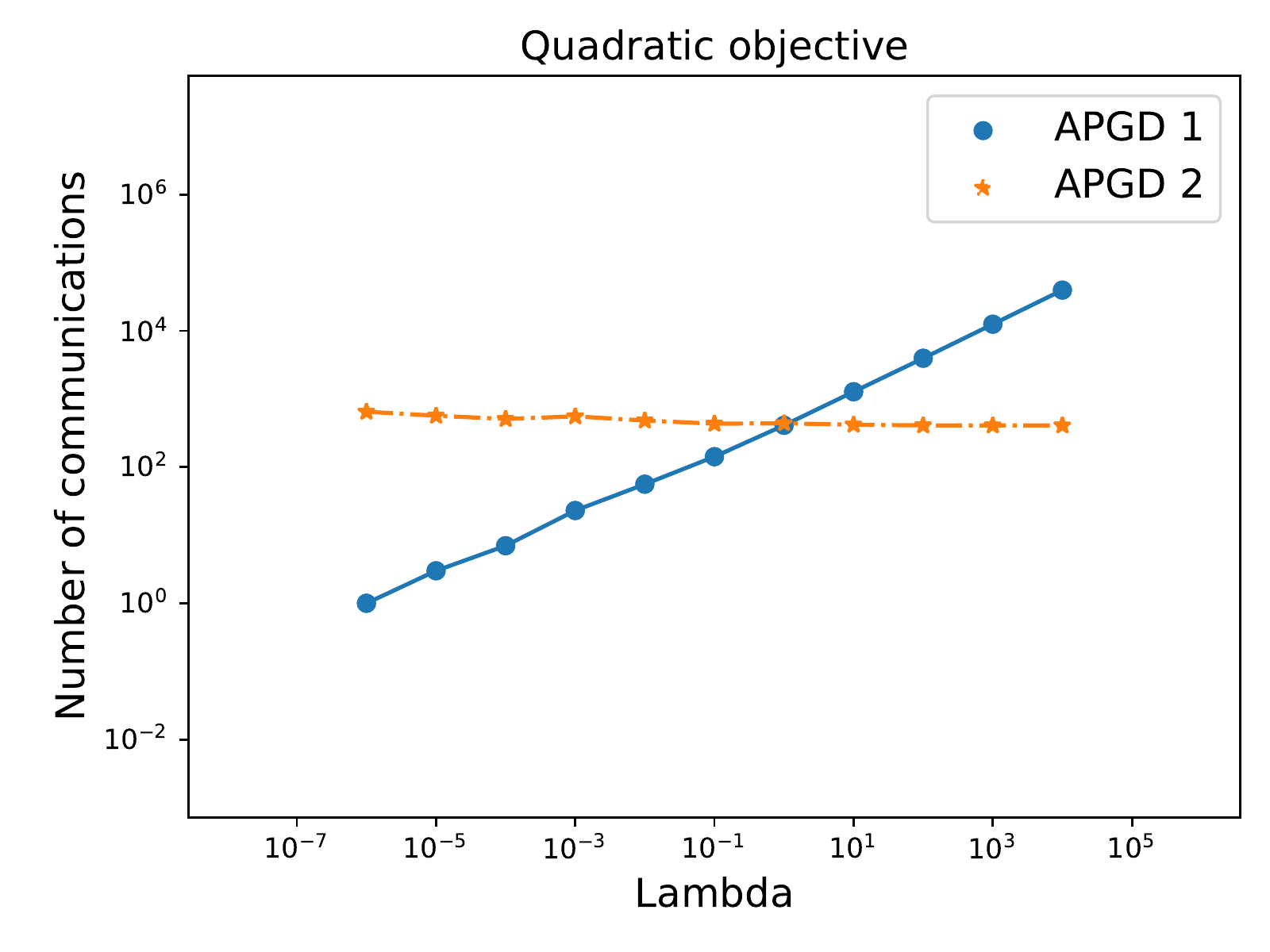}
			}
			\caption{Effect of the parameter $ \lambda$ on the communication complexity of \texttt{APGD1} and \texttt{APGD2}. For each value of $\lambda$ the y-axis indicates the number of communication required to get $10^4$-times closer to the optimum compared to the starting point. Quadratic objective with $n=50$, $d=50$. 
			}
			\label{fig:fistas}
		\end{center}
	\end{figure}

\paragraph{Experimental setup}

	In this section, we provide additional experiments comparing introduced algorithms on logistic regression with \texttt{LIBSVM} data.\footnote{Logistic regression loss for on the $j$-th data point $a_j\in \R^d$ is defined as $\phi_j(x) = \log\left(1+\exp \left(b_j a_j^\top x\right)\right) + \tfrac{\lambda}{2}\norm{x}^2_2$, where $b_j \in \{-1,1 \}$ is the corresponding label. } The local objectives are constructed by evenly dividing to the workers. We vary the parameters $m,n$ among the datasets as specified in Table \ref{tbl:dataset_division}.

We consider two types of assignment of data to the clients: \emph{homogeneous} assignment, where local data are assigned uniformly at random and \emph{heterogeneous} assignment, where we first sort the dataset according to labels, and then assign it to the clients in the given order. The heterogeneous assignment is supposed to better simulate the real-world scenarios. Next, we normalize the data $a_1, a_2, \dots,$ so that $\flocc_{i,j}$ is $1$-smooth and set $\mu = 10^{-4}$. 

For each dataset we select rather small value of $\lambda$, specifically $\lambda = \frac1m$. Lastly, for  \texttt{L2SGD+} and {\tt AL2SGD+}, we choose $p=\rho=1/m$, which is in the given setup optimal up to a constant factor in terms of the communication. We run the algorithms for $10^3$ communication rounds and track relative suboptimality\footnote{Relative suboptimality means that for iterates $\left\{ x^k \right\}_{k=1}^K$ we plot $\left\{ \tfrac{f(x^k)-f(x^\star)}{f(x^0)-f(x^\star)} \right\}_{k=1}^{K}$.} after each aggregation. Similarly to Figure~\ref{fig:com_met}, we plot relative suboptimality agains the number of communication rounds and local gradients computed. 

The remaining parameters are selected according to theory for each algorithm with one exception: For \texttt{IAPGD+Katyusha} we run {\tt Katyusha} as a local subsolver at the iteration $k$ for 
\[
\sqrt{\frac{m(L+\lambda)}{\mu+\lambda}} + 
\sqrt{\frac{m \mu (L+\lambda)}{\lambda(\mu+\lambda)}}k
    \]
iterations (slightly smaller than what our theory suggests). 


	\begin{table}
		\begin{center}
			\small
			\begin{tabular}{|c|c|c|c|c|c|}
				\hline
				{\bf Dataset}& $n$ &$m$ &$d$ & $\lambda$ & $p=\rho$ \\
				\hline
				\hline
				\texttt{a1a} & 5 & 321 & 119 & 0.003 & 0.003  \\
				\hline
				\texttt{duke} & 11 & 4 & 7129 & 0.333 & 0.250 \\
				\hline
				\texttt{mushrooms} & 12 & 677 & 112 & 0.001 & 0.001 \\
				\hline
				\texttt{madelon} & 200 & 10 &  500 & 0.111 & 0.100\\
				\hline
	            \texttt{phishing} & 335 & 33 & 68  & 0.031 & 0.030\\
            	\hline 
			\end{tabular}
		\end{center}
		\caption{Number of workers and local functions on workers for different datasets for Figures~\ref{fig:com_met} and~\ref{fig:stoch_hetero}.}
		\label{tbl:dataset_division}
	\end{table}

	\clearpage

	\bibliography{literature}
	\bibliographystyle{plain}

	\clearpage

	\appendix

	\section{Table of frequently used notation}
	To enhance the reader's convenience when navigating, we here reiterate our notation:
	
	\begin{table}[H]
		\begin{center}
			{
				\caption{Summary of frequently used notation.}
				\begin{tabular}{|c|l|c|}
					\hline
					\multicolumn{3}{|c|}{{\bf General} }\\
					\hline
					$F: \R^{nd}\rightarrow \R$ & Global objective & \eqref{eq:main}\\
					$f_i: \R^{n}\rightarrow \R$ & Local loss on $i$-th node & \eqref{eq:main} \\
					$x_i\in \R^d$ & Local model on $i$-th node & \eqref{eq:main} \\
					$x\in \R^{nd}$ & Concatenation of local models $x = [x_1, x_2, \dots, x_n]$ & \eqref{eq:main} \\
					$f: \R^{nd}\rightarrow \R$ & Average loss over nodes $f(x) \eqdef \nicefrac{1}{n}\sum_{i=1}^n f_i(x_i)$ & \eqref{eq:main} \\
					$ \psi : \R^{nd}\rightarrow \R $ & Dissimilarity penalty $ \psi(x) \eqdef \frac{1}{2 n}\sum \limits_{i=1}^n \norm{x_i-\bar{x}}^2$ &   \eqref{eq:main} \\
					$\lambda \geq 0$ & Weight of dissimilarity penalty & \eqref{eq:main}\\
					$\text{Loc}(x_i,i)$ & local oracle: \{ proximal, gradient, summand gradient\} & Sec.~\ref{sec:lower} \\
					$\mu \geq 0$& Strong convexity constant of each $f_i$ ($\tilde{f}_{i,j}$)   &  \\
					$L \geq 0$ &  Smoothness constant of each $f_i$  &  \\
					prox & Proximal operator  &  \eqref{eq:pgd} \\
					$m \geq 1$ & Number of local summands of $i$-th local loss $f_i = \nicefrac{1}{m}\sum_{j=1}^m \tilde{f}_{i,j}$  & Sec.~\ref{sec:iapgd} \\
					$\tilde{f}_{i,j}:  \R^{n}\rightarrow \R$ & $j$-th summand of $i$-th local loss, $1 \leq j \leq m$ &  Sec.~\ref{sec:iapgd} \\
					$\tilde{L} \geq 0$ & Smoothness constant of each $\tilde{f}_{i,j}$ & Sec.~\ref{sec:iapgd}\\
					$\varepsilon \geq 0$ & Precision & \\
					$x^0\in \R^{nd}$ & Algorithm initialization  & \\
					$x^\star\in \R^{nd}$ & Optimal solution of  \eqref{eq:main},  $x^\star = [x_1^\star, x_2^\star, \dots, x_n^\star]$  & \\
					$F^\star \in \R$ & Function value at minimum, $F^\star = F(x^\star)$&  \\
					\hline
					\multicolumn{3}{|c|}{{\bf Algorithms}}\\
					\hline
					{\tt APGD1}& Accelerated Proximal Gradient Descent (Algorithm~\ref{alg:fista}) & Sec.~\ref{sec:apgd_simple}\\
					{\tt APGD2}& Accelerated Proximal Gradient Descent (Algorithm~\ref{alg:fista_2}) & Sec.~\ref{sec:apgd_simple}\\
					{\tt IAPGD}& Inexact Accelerated Proximal Gradient Descent (Algorithm~\ref{alg:fista_inex}) &  Sec.~\ref{sec:iapgd}\\
					{\tt IAPGD + AGD}& {\tt IAPGD} with {\tt AGD} as a local sobsolver &  Sec.~\ref{sec:iapgd}\\
					{\tt IAPGD + Katyusha}&{\tt IAPGD} with {\tt Katyusha} as a local subsolver  &  Sec.~\ref{sec:iapgd}\\
					{\tt AL2SGD+ }& Accelerated Loopless Local Gradient Descent (Algorithm~\ref{alg:acc_stoch})& Sec.~\ref{sec:al2sgd+} \\
					$p, \rho$  & Probabilities;  parameters of {\tt AL2SGD+ }  & Sec.~\ref{sec:apgd_missing} \\
					\hline
				\end{tabular}
			}
			
		\end{center}
		\label{tbl:notation}
	\end{table}

	\clearpage

	\section{Missing parts for Section~\ref{sec:upperbound} \label{sec:apgd_missing}}
	
	In this section, we state the algorithms that were mentioned in the main paper: \texttt{APGD1} as Algorithm~\ref{alg:fista}, \texttt{APGD2} as Algorithm~\ref{alg:fista_2} and {\tt AL2SGD+} as Algorithm~\ref{alg:acc_stoch}. Next, we state the convergence rates of \texttt{APGD1},  \texttt{APGD2} as Proposition~\ref{prop:fista} and Proposition~\ref{prop:fista2} respectively. Lastly, we justify~\eqref{eq:pgd_specialized} via Lemma~\ref{lem:mnadjnjks}.
	
	\begin{proposition}\cite{beck2017first}\label{prop:fista}
		Let $\{x^k\}_{k=0}^\infty$ be a sequence of iterates generated by Algorithm~\ref{alg:fista}. Then, we have for all $k\geq 0$:
		\[
		F(x^k) -  F^\star \leq \left( 1- \sqrt{\frac{\mu}{\lambda + \mu}}\right)^k \left(
		F(x^0) - F^\star + \frac{\mu}{2n}\| x^0-x^\star\|^2
		\right).
		\]
		
	\end{proposition}

	\begin{algorithm}[h]
		\caption{{\tt APGD1}}
		\label{alg:fista}
		\begin{algorithmic}
			\REQUIRE Starting point $y^0 = x^0\in\R^{nd}$
			\FOR{ $k=0,1,2,\ldots$ }
			\STATE{ {\color{blue}Central server}  computes the average $\bar{y}^k = \frac1n \sum_{i=1}^n y^k_i$}
			\STATE For all {\color{red} clients} $i=1,\dots,n$: 
			\STATE \quad Solve the regularized local problem $x^{k+1}_i = \argmin_{z\in \R^d} f_i(z) + \frac{\lambda}{2} \| z -\bar{y}^k \|^2$  
			\STATE \quad Take the momentum step $y^{k+1}_i = x^{k+1}_i +\frac{ \sqrt{\lambda}- \sqrt{\mu}}{ \sqrt{\lambda}+ \sqrt{\mu}} ( x^{k+1}_i - x^{k}_i ) $
			\ENDFOR
		\end{algorithmic}
	\end{algorithm}
	
	\begin{proposition}\cite{beck2017first}\label{prop:fista2}
		Let $\{x^k\}_{k=0}^\infty$ be a sequence of iterates generated by Algorithm~\ref{alg:fista_2}. Then, we have for all $k\geq 0$:
		\[
		F(x^k) -  F^\star \leq \left( 1- \sqrt{\frac{\mu}{L + \mu}}\right)^k \left(
		F(x^0) - F^\star + \frac{\mu}{2n}\| x^0-x^\star\|^2
		\right).
		\]
		
	\end{proposition}

	\begin{algorithm}[h]
		\caption{{\tt APGD2}}
		\label{alg:fista_2}
		\begin{algorithmic}
			\REQUIRE Starting point $y^0 = x^0\in\R^{nd}$
			\FOR{ $k=0,1,2,\ldots$ }
			\STATE For all {\color{red} clients} $i=1,\dots,n$: 
			\STATE \quad Take a local gradient step $\tilde{y}^k_i = y^k_i -\frac1L \nabla f_i(y^k_i)$
			\STATE{ {\color{blue}Central server} computes the average $\bar{y}^k = \frac1n \sum_{i=1}^n \tilde{y}^k_i$}
			\STATE For all {\color{red} clients} $i=1,\dots,n$: 
			\STATE \quad  Take a prox step w.r.t $\lambda \psi$: $x^{k+1}_i = \frac{L\tilde{y}^k_i + \lambda \bar{y}^k}{L+\lambda}$  
			\STATE \quad  Take the momentum step $y^{k+1}_i = x^{k+1}_i +\frac{ \sqrt{\frac{L}{\mu}}-1}{ \sqrt{\frac{L}{\mu}}+1} ( x^{k+1}_i - x^{k}_i ) $
			\ENDFOR
		\end{algorithmic}
	\end{algorithm}

	\begin{lemma}\label{lem:mnadjnjks}
		Let \begin{equation} \label{eq:pgd_}
		x^{k+1} = \prox_{\frac{1}{L_h }\phi}\left( x^k - \frac{1}{L_h} \nabla h(x^k)\right),
		\end{equation}
		for  $h (x)\eqdef \lambda \psi(x) + \frac{\mu}{2n}\| x\|^2$ and $\phi(x) \eqdef f(x) -  \frac{\mu}{2n}\| x\|^2$. Then, we have
		\[
		x^{k+1}_i  = \prox_{\frac{1}{\lambda}f_i}(\bar{x}^k).
		\]
		Further, the iteration complexity of the above process is $\cO\left(\frac{\lambda}{\mu} \log\frac1\varepsilon\right)$.
	\end{lemma}
	\begin{proof}
		
		Since function $\psi$ is $\frac1n$-smooth and $(\nabla \psi(x))_i = \frac1n(x_i -\bar{x})$~\cite{hanzely2020federated}, we have $L_h = \frac{\lambda +\mu}{n}, (\nabla h(x))_i =  \frac{\lambda}{n}(x_i -\bar{x}) + \frac{\mu}{n} x_i$ and thus 
		\begin{align*}
		x^{k+1}_i &= \argmin_{z\in \R^d} \frac1n f_i(z) -  \frac{\mu}{2n}\| z\|^2 + \frac{\lambda + \mu}{2n} \left\| z - \left(x_i^k - \frac{n}{\lambda+\mu} \left(\frac{\lambda}{n}(x_i^k-\bar{x}^k )+ \frac{\mu}{n}x^k_i \right) \right) \right\|^2 \\
		&
		=
		\argmin_{z\in \R^d}f_i(z) -  \frac{\mu}{2}\| z\|^2
		+ \frac{\lambda + \mu}{2} \left\| z - \frac{\lambda}{\lambda+\mu} \bar{x}^k  \right\|^2
		\\
		&=
		\argmin_{z\in \R^d}f_i(z) + \frac{\lambda }{2} \left\| z -  \bar{x}^k  \right\|^2 = \prox_{\frac{1}{\lambda}f_i}(\bar{x}^k).
		\end{align*}
		Let us now discuss the convergence rate. Given that function $h$ is $\mu_{h}$-strongly convex, iteration complexity of~\eqref{eq:pgd} to reach $\varepsilon$-suboptimality is $\cO\left(\frac{L_h}{\mu_{h}} \log\frac1\varepsilon\right)$. Since $L_h = \frac{\lambda + \mu}{n}$ (note that $\psi$ is $\frac{1}n$ smooth~\cite{hanzely2020federated}) and $\mu_h = \frac{\mu}{n}$, the iteration complexity of the process~\eqref{eq:pgd_specialized} becomes $\cO\left(\frac{\lambda}{\mu} \log\frac1\varepsilon\right)$, as desired.
	\end{proof}

	\begin{algorithm}[h]
		\caption{{\tt AL2SGD+}}
		\label{alg:acc_stoch}
		\begin{algorithmic}
			\REQUIRE $0< \theta_1, \theta_2 <1$, $\eta, \beta , \gamma > 0$, $\probx, \proby \in (0,1)$, $y^0 = z^0 = x^0 =w^0\in \R^{nd}$
			\FOR{$k=0,1,2,\ldots$}
			\STATE For all {\color{red} clients} $i=1,\dots,n$: 
			\STATE \hskip .3cm $x^k_i = \theta_1 z^k_i + \theta_2 w^k_i + ( 1 -\theta_1 -\theta_2) y^k_i$
			\STATE $\xi = 1$ with probability $\proby$ and $0$ with probability $1-\proby$
			\IF {$\xi=0$}
			\STATE For all {\color{red} clients} $i=1,\dots,n$: 
			\STATE \hskip .3cm $\ggg^k_\tR = \frac{1}{\TR(1-\proby)} \left(\nabla \flocc_{i,j}(x^k_{\tR})- \nabla \flocc_{i,j}(w^k_{\tR})\right) + \frac1n \nabla f_i(w_i^k)  +   \frac{\lambda}{n} (w_i^k - \bar{w}^k) $
			\STATE \hskip .3cm $y^{k+1}_\tR =  x^k_i - \eta g_i^k$
			\ELSE
			\STATE{\color{blue}Central server} computes the average $\bar{x}^k = \frac{1}{n}\sum_{i=1}^n x_i^k$ and sends it back to the clients
			\STATE For all {\color{red} clients} $i=1,\dots,n$: 
			\STATE \hskip .3cm $\ggg^k_\tR=   \frac{\lambda }{\TR \proby} (x^k_{\tR} - \bar{x}^k) - \frac{(\proby^{-1} -1) \lambda}{ \TR} (w_i^k - \bar{w}^k) + \frac{1}{\TR} \nabla f_i(w^k_i)$
			\STATE \hskip .3cm Set $y_{\tR}^{k+1} =x_{\tR}^k  - \eta g_i^k $
			\ENDIF
			\STATE For all {\color{red} clients} $i=1,\dots,n$: 
			\STATE \hskip .3cm $z^{k+1}_i = \beta z^k_i + (1-\beta)x^k_i + \frac{\gamma}{\eta}(y^{k+1}_i - x^k_i)$
			\STATE    $\xi' = 1$ with probability $\probx$ and $0$ with probability $1-\probx$
			\IF {$\xi'=0$}
			\STATE For all {\color{red} clients} $i=1,\dots,n$: 
			\STATE \hskip .3cm $w^{k+1}_i = w^{k}_i$
			\ELSE
			\STATE For all {\color{red} clients} $i=1,\dots,n$: 
			\STATE \hskip .3cm $w^{k+1}_i = y^{k+1}_i$
			\STATE \hskip .3cm Evaluate and store $\nabla f_i(w^{k+1}_i) $
			\STATE {\color{blue}Central server} computes the average $\bar{w}^{k+1} = \frac{1}{n}\sum_{i=1}^n w_i^{k+1}$ and sends it back to the clients
			\ENDIF
			\ENDFOR
		\end{algorithmic}
	\end{algorithm}
	
	\clearpage
	
	\section{Proof of Theorem~\ref{thm:lb}}
	
	In this section, we provide the proof of the Theorem \ref{thm:lb}. In order to do so, we construct a set of function $f_1, f_2, \dots, f_n$ such that for any algorithm satisfying Assumption~\ref{as:oracle} and the number of the iterations $k$, one must have $\| x^{k} - x^\star\|^2  \geq \frac 1 2 \left(1-10\max\left\{ \sqrt{\frac{\mu}{\lambda}}, \sqrt{\frac{\mu}{L-\mu}}\right \} \right)^{\comm(k)+1} \| x^0 - x^\star\|^2.$

	Without loss of generality, we consider $x^0=0\in \R^{dn}$. The rationale behind our proof goes as follows: we show that the $nd$--dimensional vector $x^k$ has ``a lot of'' zero elements while $x^\star$ does not,  and hence we might lower bound $\norm{x^k-x^\star}^2$ by $\sum_{j: (x^k)_j=0} (x^\star)_j^2$, which will be large enough. As the main idea of the proof is given, let us introduce our construction.
	
	Let $d=2T$ for some large $T$ and define the local objectives as follows for even $n$
	{
		\footnotesize
		\begin{eqnarray*}
			f_1(y) = f_2(y) = \dots = f_{n/2}(y) &\eqdef&  \frac{\mu}{2}\| y \|^2  + ay_1  + \frac{\lambda}{2}c\left( \sum_{i=1}^{ T-1} (y_{2i} - y_{2i +1})^2  \right)+ \frac{\lambda b}{2}  y_{2T}^2  \\
			f_{n/2+1}(y) =  f_{n/2+2}(y) =\dots = f_{n}(y)&\eqdef& \frac{\mu}{2}\| y \|^2  +\frac{ \lambda}{2} c \left(\sum_{i=0}^{ T-1} (y_{2i + 1} - y_{2i+2})^2 \right)
		\end{eqnarray*}
	}
	and as
	{
		\footnotesize
		\begin{eqnarray*}
			f_1(y) = f_2(y) = \dots = f_{\nhalf}(y) &\eqdef&  \frac{\nhalf+1}{\nhalf}\frac{\mu}{2}\| y \|^2  + ay_1  + \frac{\lambda}{2}\frac{\nhalf+1}{\nhalf}c\left( \sum_{i=1}^{ T-1} (y_{2i} - y_{2i +1})^2  \right)+ \frac{\lambda b}{2}  y_{2T}^2  \\
			f_{\nhalf+1}(y) =  f_{\nhalf+2}(y) = \dots = f_{n}(y)&\eqdef& \frac{\mu}{2}\| y \|^2  +\frac{ \lambda}{2} c \left(\sum_{i=0}^{ T-1} (y_{2i + 1} - y_{2i+2})^2 \right)
		\end{eqnarray*}
	}
	for $n=2\nhalf+1, \nhalf\geq 1$. Note that the smoothness of the objective is now effectively controlled by parameter $c$.
	
	With such definition of functions $f_i(x_i)$, our objective is quadratic and can be written as
	\begin{equation}\label{eq:dhjabhsudga}
	\frac{n}{\lambda}F(x) = \frac12x^\top \mM x + \frac{a}{\lambda}x_1,
	\end{equation} where $\mM$ is matrix dependent on parity of $n$. For even $n$, we have
	\begin{align*}
	\mM& \eqdef \left( \mI - \frac1n \ones \ones^\top \right) \otimes \mI +\frac{ \mu}{\lambda} \mI  + \begin{pmatrix}
	\mM_1 & 0 \\
	0 & \mM_2
	\end{pmatrix},  \text{ where}
	\\
	\mM_1 & \eqdef   \mI \otimes
	\begin{pmatrix}
	0 & 0 & 0& \dots &0  \\
	0 & \mat & 0 &\ddots  &\vdots\\
	0 & 0 & \mat &\ddots  &\vdots \\
	\vdots & \ddots & \ddots & \ddots& \vdots \\
	0 & \dots & \dots & \dots & b
	\end{pmatrix} \text{ and}
	\\
	\mM_2 & \eqdef  \mI \otimes
	\begin{pmatrix}
	\mat& 0 &  \dots  \\
	0 & \mat &\dots \\
	\vdots & \vdots &\ddots
	\end{pmatrix}.
	\end{align*}
	
	When $n$ is odd, we have
	\begin{align*}
	\mM& \eqdef \left( \mI - \frac1n \ones \ones^\top \right) \otimes \mI +\frac{ \mu}{\lambda} \mI  + \begin{pmatrix}
	\mM_1 +\frac{\mu}{\nhalf\lambda} \mI& 0 \\
	0 & \mM_2
	\end{pmatrix}, \text{ where}
	\\
	\mM_1 & \eqdef   \mI \otimes
	\begin{pmatrix}
	0 & 0 & 0& \dots &0  \\
	0 & \mmat & 0 &\ddots  &\vdots\\
	0 & 0 & \mmat &\ddots  &\vdots \\
	\vdots & \ddots & \ddots & \ddots& \vdots \\
	0 & \dots & \dots & \dots & b
	\end{pmatrix} \text{ and}
	\\
	\mM_2 & \eqdef \mI \otimes
	\begin{pmatrix}
	\mat& 0 &  \dots  \\
	0 & \mat &\dots \\
	\vdots & \vdots &\ddots
	\end{pmatrix}.
	\end{align*}
	
	Note that our functions $f_k$ depends on parameters $a\in \R,b,c\in \R^+$. We will choose these parameters later in the way that the optimal solution can be obtained easily.
	
	Now let's discuss optimal model for the objective. Since the the objective is strongly convex, the optimum $x^\star$ is unique. Let us find what it is. For the sake of simplicity, denote $y^\star \eqdef x_1^\star, z^\star =x^\star_{n}$. Due to the symmetry, we must have
	\[y^\star = x_2^\star = \dots x_{n/2}^\star , \quad  z^\star = x_{n/2+1} = x_{n/2+2} = \dots x_{n-1}^\star \qquad \text{for even }n
	\]
	and
	\[y^\star =x_2^\star = \dots x_{\nhalf}^\star,
	\quad z^\star = x_{\nhalf+1} = \dots = x_{n-1}^\star \qquad \text{for odd }n.
	\]
	Now we use the following lemma to express elements of $y^\star, z^\star$ recursively.
	
	\begin{lemma} \label{le:one_step_even}
		Let \[
		w_i \eqdef
		\begin{cases}
		\begin{pmatrix}
		z^\star_{i}\\
		y^\star_{i}
		\end{pmatrix} & \text{if } i \text{ is even} \\
		\begin{pmatrix}
		y^\star_{i}\\
		z^\star_{i}
		\end{pmatrix} & \text{if } i \text{ is odd}
		\end{cases}.
		\]
		Then, we have
		\[
		w_{i+1}
		=
		\mQ
		w_i
		\]
		where
		\[
		\mQ \eqdef \begin{pmatrix}
		-\frac{r}{c} &  \frac{c + \frac{\mu}{\lambda} +r}{c} \\
		- \frac{c + \frac{\mu}{\lambda} +r}{c}&  \frac{\left(c + \frac{\mu}{\lambda} + r\right)^2}{cr} - \frac{c}{r}
		\end{pmatrix}
		\]
		and
		\[
		r = \begin{cases}
		\frac12 & \text{if } n \text{ is even} \\
		\frac{\nhalf}{n} & \text{if } n \text{ is odd}
		\end{cases} .
		\]
	\end{lemma}
	
	To prove the lemma, we shall manipulate the first-order optimality conditions of~\eqref{eq:dhjabhsudga}. 
	
	\begin{proof}
		\textbf {For even $n$}, the first-order optimality conditions yield
		
		\begin{eqnarray}
		\left(  c+\frac{1}2 + \frac{\mu}{\lambda} \right) z^\star_{2i-1} - cz^\star_{2i} - \frac{1}{2}y^\star_{2i-1} = 0 && \text{for } 1\leq i \leq T  \label{eq:z1even}\\
		\left(  c+\frac{1}2+ \frac{\mu}{\lambda} \right) z^\star_{2i} -  cz^\star_{2i-1} - \frac{1}{2}y^\star_{2i} = 0 && \text{for } 1\leq i \leq T  \label{eq:z2even}\\
		\left( c+\frac{1}2+ \frac{\mu}{\lambda} \right) y^\star_{2i} - cy^\star_{2i+1} - \frac12 z^\star_{2i} = 0 && \text{for } 1\leq i \leq T-1  \label{eq:y1even}\\
		\left( c+\frac{1}2 + \frac{\mu}{\lambda} \right) y^\star_{2i+1} -  cy^\star_{2i} - \frac12 z^\star_{2i+1} = 0 && \text{for } 1\leq i \leq T-1  \label{eq:y2even}
		\end{eqnarray}
		
		Equalities~\eqref{eq:z1even} and~\eqref{eq:z2even} can be equivalently written as
		
		\begin{equation}\label{eq:ehquiwhu}
		\begin{pmatrix}
		c & 0 \\
		- c-r- \frac{\mu}{\lambda}  & r  \\
		\end{pmatrix}
		\begin{pmatrix}
		z^\star_{2i}\\
		y^\star_{2i}
		\end{pmatrix}
		=
		\begin{pmatrix}
		c+r + \frac{\mu}{\lambda}  & - r \\
		-c& 0 \\
		\end{pmatrix}
		\begin{pmatrix}
		z^\star_{2i-1}\\
		y^\star_{2i-1}
		\end{pmatrix} \qquad \text{for } 1\leq i \leq T
		\end{equation}
		
		and consequently we must have for all $1\leq i \leq T$
		\begin{eqnarray*}
			\begin{pmatrix}
				z^\star_{2i}\\
				y^\star_{2i}
			\end{pmatrix}
			&=&
			\begin{pmatrix}
				c& 0 \\
				- c-r - \frac{\mu}{\lambda}  &  r  \\
			\end{pmatrix}^{-1}
			\begin{pmatrix}
				c+r + \frac{\mu}{\lambda}  & -r  \\
				-c& 0 \\
			\end{pmatrix}
			\begin{pmatrix}
				z^\star_{2i-1}\\
				y^\star_{2i-1}
			\end{pmatrix}
			\\
			&=&
			\begin{pmatrix}
				\frac{c + \frac{\mu}{\lambda} + r}{c}&    -\frac{r}{c} \\
				\frac{\left(c + \frac{\mu}{\lambda} + r\right)^2}{rc} - \frac{c}{r} & -\frac{c + \frac{\mu}{\lambda} + r}{c}
			\end{pmatrix}
			\begin{pmatrix}
				z^\star_{2i-1}\\
				y^\star_{2i-1}
			\end{pmatrix}
			\\
			&=&
			\mQ
			\begin{pmatrix}
				y^\star_{2i-1}\\
				z^\star_{2i-1}
			\end{pmatrix}.
		\end{eqnarray*}
		
		Analogously, from~\eqref{eq:y1even} and~\eqref{eq:y2even} we deduce that for all $1\leq i \leq T-1$
		\[
		\begin{pmatrix}
		y^\star_{2i+1}\\
		z^\star_{2i+1}
		\end{pmatrix}
		=
		\mQ
		\begin{pmatrix}
		z^\star_{2i}\\
		y^\star_{2i}
		\end{pmatrix}.
		\]

		\textbf{For odd $n$}, the first-order optimality conditions yield
		{
			\footnotesize
			\begin{eqnarray}
			\left(  c+\frac{\nhalf}{n} + \frac{\mu}{\lambda} \right) z^\star_{2i-1} - cz^\star_{2i} - \frac{\nhalf}{n}y^\star_{2i-1} = 0 && \text{for } 1\leq i \leq T  \label{eq:z1odd}\\
			\left(  c+\frac{\nhalf}{n}+ \frac{\mu}{\lambda} \right) z^\star_{2i} -  cz^\star_{2i-1} - \frac{\nhalf}{n}y^\star_{2i} = 0 && \text{for } 1\leq i \leq T  \label{eq:z2odd}\\
			\left( \frac{\nhalf+1}{\nhalf}c+\frac{\nhalf+1}{n}+ \frac{\nhalf+1}{\nhalf}\frac{\mu}{\lambda} \right) y^\star_{2i} - \frac{\nhalf+1}{\nhalf}cy^\star_{2i+1} - \frac{\nhalf+1}{n} z^\star_{2i} = 0 && \text{for } 1\leq i \leq T-1  \label{eq:y1odd}\\
			\left( \frac{\nhalf+1}{\nhalf}c+\frac{\nhalf+1}{n} +\frac{\nhalf+1}{\nhalf} \frac{\mu}{\lambda} \right) y^\star_{2i+1} -  \frac{\nhalf+1}{\nhalf}cy^\star_{2i} - \frac{\nhalf+1}{n} z^\star_{2i+1} = 0 && \text{for } 1\leq i \leq T-1  \label{eq:y2odd}
			\end{eqnarray}
		}

		Equalities~\eqref{eq:z1odd} and~\eqref{eq:z2odd} can be equivalently written as
		
		\[
		\begin{pmatrix}
		c & 0 \\
		- c-r- \frac{\mu}{\lambda}  & r  \\
		\end{pmatrix}
		\begin{pmatrix}
		z^\star_{2i}\\
		y^\star_{2i}
		\end{pmatrix}
		=
		\begin{pmatrix}
		c+r + \frac{\mu}{\lambda}  & - r \\
		-c& 0 \\
		\end{pmatrix}
		\begin{pmatrix}
		z^\star_{2i-1}\\
		y^\star_{2i-1}
		\end{pmatrix} \qquad \text{for } 1\leq i \leq T,
		\]
		which is identical to~\eqref{eq:ehquiwhu}, and thus
		
		\begin{eqnarray*}
			\begin{pmatrix}
				z^\star_{2i}\\
				y^\star_{2i}
			\end{pmatrix}
			&=&
			\mQ
			\begin{pmatrix}
				y^\star_{2i-1}\\
				z^\star_{2i-1}
			\end{pmatrix}.
		\end{eqnarray*}

		Similarly, ~\eqref{eq:y1odd} and~\eqref{eq:y2odd} imply that for all $1\leq i \leq T-1$
		\[
		\begin{pmatrix}
		y^\star_{2i+1}\\
		z^\star_{2i+1}
		\end{pmatrix}
		=
		\mQ
		\begin{pmatrix}
		z^\star_{2i}\\
		y^\star_{2i}
		\end{pmatrix}.
		\]
		
	\end{proof}
	
	As consequence of Lemma \ref{le:one_step_even}, we have that $w_k = \mQ ^{k-1}w_1$ with $\frac13 \leq r\leq \frac12$. 
	Now we use the flexibility to choose $a \in \R, b \in \R^+$, so that we can find suitable $w_k$ (and thus suitable $x^\star$). Specifically, we aim to choose $a,b$, so that $w_1$ will be the eigenvector of $\mQ$, corresponding to a suitable eigenvalue $\gamma$ of  matrix $\mQ$. Then $w_k$ could be written as $w_k = \gamma^k w_1$.
	
	\begin{lemma}
		Choose $c\eqdef \begin{cases}
		1 & \text{if } L\geq \lambda + \mu \\
		\delta\frac{\mu}{\lambda}, \delta \geq 1  & \text{if } L< \lambda + \mu
		\end{cases}
		$ and
		\begin{equation}
		b \eqdef
		\begin{cases}
		\frac{\frac{\mu^2}{\lambda^2} + 2\frac{\mu}{\lambda} + 2r +2r\frac{\mu}{\lambda}  + 2r^2+ \left( \frac{\mu}{\lambda} (\frac{\mu}{\lambda} +2r)(\frac{\mu}{\lambda}+2) (\frac{\mu}{\lambda} +2r+2)\right)^\frac12 }{2r(1+\frac{\mu}{\lambda}+r)} -1-\frac{\mu}{\lambda}  & \text{if } L\geq \lambda + \mu
		\\
		\frac{ \frac{\mu^2}{\lambda^2} + 2r^2+ 2 r \frac{\mu}{\lambda}+ 2\delta r \frac{\mu}{\lambda} +2\delta  \frac{\mu^2}{\lambda^2} +
			\frac{\mu}{\lambda} \left( (2\delta+1)(\frac{\mu}{\lambda}+2r) (\frac{\mu}{\lambda} +2r+2\delta \frac{\mu}{\lambda})\right)^\frac12}
		{2 r (\frac{\mu}{\lambda} + r + \delta \frac{\mu}{\lambda})}
		-1-\frac{\mu}{\lambda}  & \text{if } L< \lambda + \mu
		\end{cases} .
		\label{eq:bdef}
		\end{equation}
		
		Then, we have $b\geq 0$ and
		\[
		w_i = \gamma^{i-1} w_1 \neq \begin{pmatrix}
		0\\0
		\end{pmatrix} \qquad \text{for } i=1,2,\dots, d
		\]
		where
		\begin{equation}\label{eq:gamma_ineq}
		\gamma\eqdef
		\begin{cases}
		\frac{\frac{\mu^2}{\lambda^2} + 2\frac{\mu}{\lambda} + 2r +2r\frac{\mu}{\lambda}  - \left( \frac{\mu}{\lambda} (\frac{\mu}{\lambda} +2r)(\frac{\mu}{\lambda}+2) (\frac{\mu}{\lambda} +2r+2)\right)^\frac12}{2r}\geq 1-10\sqrt{\frac{\mu}{\lambda}}
		& \text{if } L\geq \lambda + \mu
		\\
		\frac{\frac{\mu}{\lambda} + 2r+ 2\delta r +2\delta \frac{\mu}{\lambda}  - \left( (2\delta+1)(\frac{\mu}{\lambda}+2r) (\frac{\mu}{\lambda} +2r+2\delta \frac{\mu}{\lambda})\right)^\frac12}{2\delta r}
		\geq 1-10\sqrt{\frac{1}{\delta}}
		& \text{if } L< \lambda + \mu
		\end{cases}.
		\end{equation}

	\end{lemma}

	\begin{proof}
		First, note that if $c=1$, each local objective is $(\mu+\lambda)$-smooth, and thus also $L$-smooth (and therefore the choice of $c$ does not contradict the smoothness). Next, if $L \geq \lambda + \mu $, the vector
		\[
		v\eqdef \begin{pmatrix}
		\frac{\frac{\mu^2}{\lambda^2} + 2\frac{\mu}{\lambda} + 2r +2r\frac{\mu}{\lambda}  + 2r^2+ \left( \frac{\mu}{\lambda} (\frac{\mu}{\lambda} +2r)(\frac{\mu}{\lambda}+2) (\frac{\mu}{\lambda} +2r+2)\right)^\frac12 }{2r(1+\frac{\mu}{\lambda}+r)} \\
		1
		\end{pmatrix}
		\]
		is an unnormalized eigenvector of $\mQ$ corresponding to eigenvalue $\gamma$.\footnote{See a MatLab symbolic verification at file {\tt eigenvalues.m}.} Next, we prove $b\geq0$ and $\gamma \geq 1 - 10\sqrt{\frac{\mu}{\lambda}}$ using Mathematica, see the file {\tt proof.nb} and the screen shot below.
		
		\begin{figure}[H]
		\includegraphics[width=\textwidth]{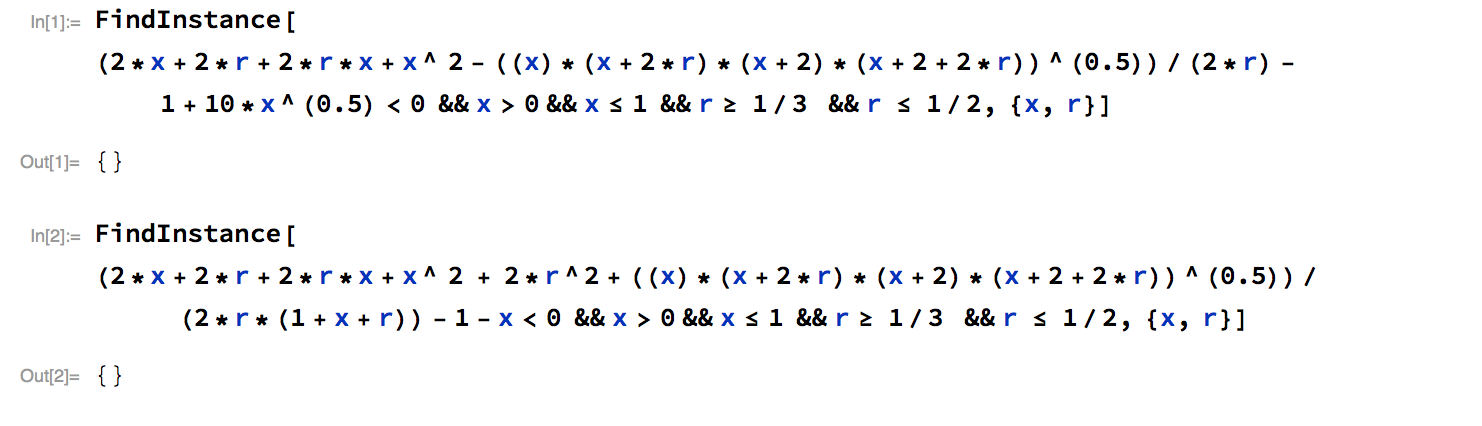}
		\end{figure}
		
		Let us look now at the case where $L \leq \lambda + \mu $. Now, the vector
		\[
		v\eqdef \begin{pmatrix}
		\frac{ \frac{\mu^2}{\lambda^2} + 2r^2+ 2 r \frac{\mu}{\lambda}+ 2\delta r \frac{\mu}{\lambda} +2\delta  \frac{\mu^2}{\lambda^2} +
			\frac{\mu}{\lambda} \left( (2\delta+1)(\frac{\mu}{\lambda}+2r) (\frac{\mu}{\lambda} +2r+2\delta \frac{\mu}{\lambda})\right)^\frac12}
		{2 r (\frac{\mu}{\lambda} + r + \delta \frac{\mu}{\lambda})}
		\\
		1
		\end{pmatrix}
		\]
		is an unnormalized eigenvector of $\mQ$ corresponding to eigenvalue $\gamma$.\footnote{See a MatLab symbolic verification at file {\tt eigenvaleus.m}.} Next, we prove $b\geq0$ and $\gamma \geq 1 - 10\sqrt{\frac{1}{\delta}}$ using Mathematica, see the file {\tt proof.nb} and the screen shot below.
		
		\begin{figure}[H]
				\includegraphics[width=\textwidth]{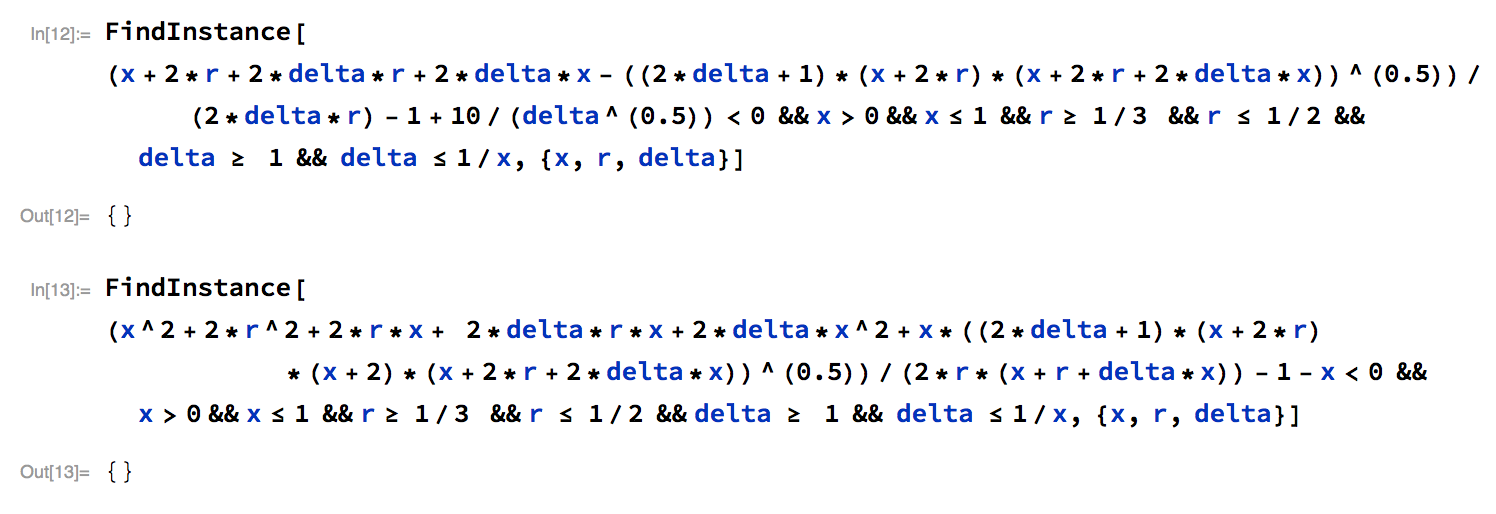}
		\end{figure}

		Setting $b$ according to~\eqref{eq:bdef} we assure that $w_i$ is a multiple of $v$ and consequently we  have
		\[
		w_i = \gamma^{i-1} w_1\qquad \text{for } i=1,2,\dots, d
		\]
		as desired. It remains to mention that $w_i\neq \begin{pmatrix}
		0\\0\end{pmatrix}$ regardless of the choice of $a \neq 0$.
		
	\end{proof}

	\begin{proof} \textbf{Theorom \ref{thm:lb}}
		
		Let $x^0 = 0\in \R^{nd}$. Note that our oracle allows us at most $K+1$ nonzero coordinates of $x^K$ after $K$ rounds of communications. Consequently,~
		\begin{eqnarray*}
			\frac{\| x^{K}-x^\star  \|^2}{\|x^0-x^\star \|^2}
			&\geq &
			\frac12
			\frac{ \sum_{j=K+2}^d  \| w_j \|^2}{\sum_{j=1}^d  \| w_j \|^2}
			=
			\frac12
			\frac{ \sum_{j=K+2}^d  \gamma^{j-1}\| w_1 \|^2}{\sum_{j=1}^d \gamma^{j-1} \| w_1 \|^2}
			=
			\frac12
			\frac{ \gamma^{K+1}  \sum_{j=0}^{d-K-2}  \gamma^{j}}{\sum_{j=0}^{d-1} \gamma^{j} }
			\\
			&=&
			\frac12
			\gamma^{K+1}\frac{ 1-\gamma^{d-K-1}}{1-\gamma^d}
			\stackrel{(*)}{\geq}
			\frac 1 4 \left(1-10\max \left\{ \sqrt{\frac{\mu}{\lambda}}, \sqrt{\frac{1}{\delta}} \right\} \right)^{K+1}
			\\
			&=&
			\frac 1 4 \left(1-10\max\left\{ \sqrt{\frac{\mu}{\lambda}}, \sqrt{\frac{\mu}{L-\mu}}\right \} \right)^{K+1}
		\end{eqnarray*}
		where the inequality $(*)$ holds for large enough $T$ (and consequently large enough $d=2T$). \QED

	\end{proof}

	\clearpage

	\section{Proofs for Section~\ref{sec:upperbound}}
	
	\subsection{Towards the Proof of Theorems~\ref{thm:inexact} and~\ref{thm:inexact_stoch}}
	
	\begin{proposition} \label{prop:fista_inexact}
		Iterates of Algorithm~\ref{alg:fista_inex} satisfy
		{
			\footnotesize
			\begin{align}
			\nonumber
			&F(x^k) -  F^\star \\
			& \leq \left( 1- \sqrt{\frac{\mu}{\lambda}}\right)^k \left(
			\sqrt{2(F(x^0) - F^\star)}
			+
			2\sqrt{\frac{\lambda}{\mu} } \left( \sum_{i=1}^k \epsilon_i^\frac12  \left( 1- \sqrt{\frac{\mu}{\lambda}}\right)^{-\frac{i}{2}}\right)
			+
			\sqrt{
				\sum_{i=1}^k \epsilon_i \left( 1- \sqrt{\frac{\mu}{\lambda}}\right)^{-i}
			}
			\right)^2.
			\label{eq:inexact_prop}
			\end{align}
		}
	\end{proposition}
	\begin{proof}
		First, notice that the objective is $\frac{\lambda}{n}$ smooth and $\frac{\mu}{n}$-strongly convex. Next, the error in the evaluation of the proximal operator at iteration $k$ can be expressed as
		\[
		\sum_{i=1}^n \frac1n f_i(x_i^{k+1}) +  \frac{\lambda}{2n}\|x_i^{k+1} \bar{y}^k  \|^2 \leq \sum_{i=1}^n \frac1n \epsilon_k = \epsilon_k.
		\]
		It remains to apply Proposition 4 from~\cite{schmidt2011convergence}.
	\end{proof}
	
	\subsubsection{General convergence rate of {\tt IAPGD}}
	Theorem~\ref{thm:inexact_main} shows that the expected number of communications that Algorithm~\ref{alg:fista_inex} requires to reach $\varepsilon$-approximate solution is $\tilde{\cO}\left( \sqrt{\frac{\lambda}{\mu}}\right)$, given that~\eqref{eq:epsilon_bound_stoch} holds.
	
	\begin{theorem}\label{thm:inexact_main}
		Assume that for all $k\geq0, 1\leq i \leq n$, the subproblem~\eqref{eq:algo_suboptimality} was  solved up to a suboptimality\footnote{See Algorithm~\ref{alg:fista_inex} for the exact meaning.} $\epsilon_k$ by a possibly randomized iterative algorithm such that
		\begin{equation}\label{eq:epsilon_bound_stoch}
		\E{\epsilon_k \mid x^k } \leq
		\left(1-\sqrt{\frac{\mu}{\lambda}} \right)^{2k} R^2
		\end{equation}
		for some fixed $R>0$. Consequently, we have
		\begin{equation}\label{eq:dajnkbdbhas}
		\E{\left(F(x^k) -  F^\star \right)^{\frac12}}
		\leq
		\left(1-\sqrt{\frac\mu\lambda} \right)^{\frac{k}2} \left(
		\sqrt{2(F(x^0) - F^\star)}
		+
		2\left( 2\sqrt{\frac{\lambda}{\mu} } +1\right)\sqrt{\frac{\lambda}{\mu}}R
		\right).
		\end{equation}
	\end{theorem}
	
	\begin{proof}

		Let $\omega \eqdef 1- \sqrt{\frac{\mu}{\lambda}}$. Proposition~\ref{prop:fista_inexact} gives us
		
		\begin{eqnarray*}
			\left(F(x^k) -  F^\star \right)^{\frac12}
			&\stackrel{\eqref{eq:inexact_prop}}{\leq}&
			\omega^{\frac{k}2} \left(
			\sqrt{2(F(x^0) - F^\star)}
			+
			2\sqrt{\frac{\lambda}{\mu} } \left( \sum_{i=1}^k \epsilon_i^\frac12 \omega^{-\frac{i}{2}}\right)
			+
			\sqrt{
				\sum_{i=1}^k \epsilon_i \omega^{-i}
			}
			\right)
			\\
			&\leq&
			\omega^{\frac{k}2} \left(
			\sqrt{2(F(x^0) - F^\star)}
			+
			\left( 2\sqrt{\frac{\lambda}{\mu} } +1\right) \left( \sum_{i=1}^k \epsilon_i^\frac12  \omega^{-\frac{i}{2}}\right)
			\right).
		\end{eqnarray*}
		
		Taking the expectation, we get
		
		\begin{eqnarray*}
			\E{\left(F(x^k) -  F^\star \right)^{\frac12}}
			&\leq&
			\omega^{\frac{k}2} \left(
			\sqrt{2(F(x^0) - F^\star)}
			+
			\left( 2\sqrt{\frac{\lambda}{\mu} } +1\right) \left( \sum_{i=1}^k \E{\epsilon_i^\frac12}  \omega^{-\frac{i}{2}}\right)
			\right)
			\\
			&\leq&
			\omega^{\frac{k}2} \left(
			\sqrt{2(F(x^0) - F^\star)}
			+
			\left( 2\sqrt{\frac{\lambda}{\mu} } +1\right) \left( \sum_{i=1}^k \E{\epsilon_i}^\frac12  \omega^{-\frac{i}{2}}\right)
			\right)
			\\
			&\stackrel{\eqref{eq:epsilon_bound_stoch}}{\leq}&
			\omega^{\frac{k}2} \left(
			\sqrt{2(F(x^0) - F^\star)}
			+
			\left( 2\sqrt{\frac{\lambda}{\mu} } +1\right) R \left( \sum_{i=1}^k  \omega^{\frac{i}{2}}\right)
			\right)
			\\
			&\leq &
			\omega^{\frac{k}2} \left(
			\sqrt{2(F(x^0) - F^\star)}
			+
			\left( 2\sqrt{\frac{\lambda}{\mu} } +1\right) R \left( \sum_{i=1}^\infty  \omega^{\frac{i}{2}}\right)
			\right)
			\\
			&= &
			\omega^{\frac{k}2} \left(
			\sqrt{2(F(x^0) - F^\star)}
			+
			\left( 2\sqrt{\frac{\lambda}{\mu} } +1\right) R \frac{\omega^{\frac12}}{1-\omega^{\frac12}}
			\right)
			\\
			&\leq &
			\omega^{\frac{k}2} \left(
			\sqrt{2(F(x^0) - F^\star)}
			+
			\left( 2\sqrt{\frac{\lambda}{\mu} } +1\right) 2R \frac{1}{1-\omega}
			\right)
			\\
			&= &
			\omega^{\frac{k}2} \left(
			\sqrt{2(F(x^0) - F^\star)}
			+
			\left( 2\sqrt{\frac{\lambda}{\mu} } +1\right) 2R\sqrt{\frac{\lambda}{\mu}}
			\right),
		\end{eqnarray*}
		which is exactly~\eqref{eq:dajnkbdbhas}.
	\end{proof}
	
	\subsubsection{Proof of Theorem~\ref{thm:inexact}}
	
	Denote $\cS' \eqdef  \{x; F(x) \leq F^\star +  8 (F(x^0)-F^\star) \} $, $\cS \eqdef  \{(2-\alpha)x' - (1-\alpha)x''; x',x''\in \cS', 0\leq \alpha \leq 1\}$ and $ D\eqdef \diam(\cS)< \infty$. Consequently,
	
	\begin{equation}\label{eq:mnadbdhsdbhakkj}
	D^2 \leq 36 \max_{x\in \cS} \|x-x^\star \|^2 \leq \frac{18n}{\mu} \max_{x\in \cS} (F(x)-F(x^\star)) \leq  \frac{144n}{\mu} (F(x^0)-F^\star)
	\end{equation}

	Let us proceed with induction. Suppose that for all $0\leq t<k$  we have \[F(x^i) -  F^\star \leq 8 \left( 1- \sqrt{\frac{\mu}{\lambda}}\right)^t (F(x^0) - F^\star).\] Consequently, $x^t \in \cS' $ for all $0\leq t<k$. Thanks to the update rule of sequence $\{y\}_{t=1}^\infty$, we must have $y^{k-1}\in \cS$. Next, define $\hat{x}_i^{k} \eqdef \argmin_{z\in \R^d} f_i(z) + \frac{\lambda}{2n} \| z -\bar{y}^{k-1} \|^2$. Clearly, $\hat{x}^{k}\in \cS$, and consequently, $\|\hat{x}^{k} - y^{k-1} \|^2\leq D^2$.
	
	We will next show that
	\begin{equation}
	\label{eq:eps_bound}
	\epsilon_k \leq R^2 \omega^{2k},
	\end{equation}
	where
	\begin{equation}\label{eq:romega_def}
	R \eqdef \frac{\sqrt{2(F(x^0) - F^\star)}}{2\sqrt{\frac{\lambda}{\mu}} \left(2 \sqrt{\frac{\lambda}{\mu}} +1\right)}, \qquad \omega \eqdef 1-\sqrt{\frac{\mu}{\lambda}}.
	\end{equation}
	
	Define $h^k_i(z)\eqdef f_i(z) + \frac{\lambda}{2n} \| z -\bar{y}^{k-1} \|^2$. Since $h_i^{k}$ is $\frac1n (L+\lambda)$ smooth and $\frac1n(\mu + \lambda)$ strongly convex,  running {\tt AGD} locally for $c_1 + c_2k$ iterations with\footnote{Inequality $(*)$ holds since for any $0\leq  a <1$ we have $\frac{-1}{\log(1-a)}\leq \frac1a$, while $(**)$ holds since $\log\left(\frac{1}{1-\sqrt{\frac{\mu}{\lambda}}} \right) \leq  2\sqrt{\frac{\mu}{\lambda}}$ thanks to $\lambda \geq 2\mu$. }
	\begin{eqnarray*}
		c_1 &\eqdef&
		-\frac{\log \frac{4LD^2}{R^2}}{ \log \left(1-\sqrt{\frac{\mu+\lambda}{L+ \lambda}} \right)}
		\stackrel{(*)}{\leq}
		\sqrt{\frac{L+ \lambda}{\mu+\lambda}} \log \frac{4LD^2}{R^2}
		\stackrel{\eqref{eq:mnadbdhsdbhakkj}}{\leq}
		\sqrt{\frac{L+ \lambda}{\mu+\lambda}} \log \frac{1152 L \lambda n \left(2 \sqrt{\frac{\lambda}{\mu} }+1 \right)^2}{\mu^2}
		,
		\\
		c_2 &\eqdef&
		\frac{2 \log \omega}{\log  \left(1-\sqrt{\frac{\mu+\lambda}{L+ \lambda}} \right) }
		\stackrel{(*)+(**)}{\leq} 4\sqrt{\frac{\mu(L+ \lambda)}{\lambda(\mu+\lambda)}}
	\end{eqnarray*}
	yields
	
	\begin{eqnarray*}
		\epsilon_k
		&\stackrel{\text{\cite{schmidt2011convergence}, Prop }4 }{\leq}&
		\left(1-\sqrt{\frac{\mu+\lambda}{L+ \lambda}} \right)^{c_1+c_2 k} 4\left( \sum_{i=i}^n \left( h^k_i(y_i^{k-1}) - h^k_i(\hat{x}^k_i)\right)\right)
		\\
		&\leq&
		\left(1-\sqrt{\frac{\mu+\lambda}{L+ \lambda}} \right)^{c_1+c_2k} 4LD^2
		\\
		&=&
		\exp\left(
		c_2k \log\left(1-\sqrt{\frac{\mu+\lambda}{L+ \lambda}}\right) + c_1 \log \left(1-\sqrt{\frac{\mu+\lambda}{L+ \lambda}} \right)+ \log \left(4LD^2\right)\right)
		\\
		&=&
		\exp\left(
		2k \log\omega +  \log (R^2) \right)
		\\
		&=&
		R^2 \omega^{2k}
	\end{eqnarray*}
	as desired.

	Next, Theorem~\ref{thm:inexact_main} gives us
	{
		\begin{eqnarray*}
			F(x^k) -  F^\star
			&\stackrel{\eqref{eq:dajnkbdbhas}}{\leq} &
			\left( 1- \sqrt{\frac{\mu}{\lambda}}\right)^k \left(\sqrt{2(F(x^0) - F^\star)} +
			2\left( 2\sqrt{\frac{\lambda}{\mu} } +1\right)\sqrt{\frac{\lambda}{\mu}}R
			\right)^2
			\\
			&\stackrel{\eqref{eq:romega_def}}{=} &
			8 \left( 1- \sqrt{\frac{\mu}{\lambda}}\right)^k (F(x^0) - F^\star),
		\end{eqnarray*}
	}
	as desired.
	
	Consequently, in order to reach $\varepsilon$ suboptimality, we shall set $k = \cO\left( \sqrt{\frac{\lambda}{\mu}}\log\frac1\varepsilon\right)$. The total number of local gradient computation thus is
	\begin{eqnarray*}
		\sum_{i=1}^k( c_1  + c_2i )&=& kc_1 + c_2 \cO(k^2)
		\\
		&=&
		\cO\left(
		\sqrt{\frac{L+ \lambda}{\mu+\lambda}} \log \frac{32 L \lambda n^2 \left(4 \sqrt{\frac{\lambda}{\mu} +1} \right)^2}{\mu^2}  \sqrt{\frac{\lambda}{\mu}}\log\frac1\varepsilon +
		\sqrt{\frac{\mu(L+ \lambda)}{\lambda(\mu+\lambda)}}
		\frac{\lambda}{\mu} \left( \log\frac1\varepsilon \right)^2
		\right)
		\\
		&=&
		\cO\left(
		\sqrt{\frac{L+ \lambda}{\mu}} \log\frac1\varepsilon
		\left( \log \frac{ L \lambda n}{\mu} +\log\frac1\varepsilon \right)
		\right).
	\end{eqnarray*}
	
	\QED

	\subsubsection{Proof of Theorem~\ref{thm:inexact_stoch}}
	
	Next, since the sequence of iterates $\{ x^k\}_{k=0}^\infty$ is bounded, so is the sequence $\{ y^k\}_{k=0}^\infty$, and consequently, the initial distance to the optimum is bounded for each local subproblem too. As the local objective is $(\Lloc + \lambda)$-smooth and $(\mu + \lambda)$-strongly convex, in order to guarantee~\eqref{eq:epsilon_bound_stoch}, {\tt Katyusha} requires
	\begin{eqnarray*}
		\cO\left( \left(m + \sqrt{m\frac{\Lloc + \lambda}{\mu + \lambda}}\right)\log\frac{1}{R^2 \omega^2k}  \right)
		&=&
		\cO\left(\left( m+  \sqrt{m\frac{\Lloc + \lambda}{\mu + \lambda}}\right)\left( \log\frac{1}{R^2} +  2k\log\frac{1}{ \omega}  \right)\right)
		\\
		&\stackrel{(***)}{=}&
		\cO\left(\left(m+  \sqrt{m \frac{\Lloc + \lambda}{\mu + \lambda}}\right)\left( \log\frac{1}{R^2} +  k\sqrt{\frac{\mu}{\lambda}} \right)\right)
	\end{eqnarray*}
	iterations.\footnote{Inequality $(***)$ holds since $\log\left(\frac{1}{1-\sqrt{\frac{\mu}{\lambda}}} \right) \leq  2\sqrt{\frac{\mu}{\lambda}}$ thanks to $\lambda \geq 2\mu$.}

	Lastly, since {\tt Katyusha} requires $\cO(1)$ local stochastic gradient evaluations on average, the total local gradient complexity becomes
	\begin{eqnarray*}
		\sum_{t=1}^{\tilde{\cO}\left(\sqrt{\frac{\lambda}{\mu}}\right)} \cO\left(\left(m+  \sqrt{m \frac{\Lloc + \lambda}{\mu + \lambda}}\right)\left( \log\frac{1}{R^2} +  t\sqrt{\frac{\mu}{\lambda}} \right)\right)
		=
		\tilde{\cO}\left(
		\left(m\sqrt{\frac{\lambda}{\mu}}
		+  \sqrt{m \frac{\Lloc + \lambda}{\mu }}\right)
		\right).
	\end{eqnarray*}

	\QED

	\subsection{Towards the Proof of Theorem~\ref{thm:a2} \label{sec:a2_proof}}
	
	\begin{lemma}\label{lem:es_stoch}
		Suppose that $\flocc_{ij}$ is $\Lloc$ smooth for all $1\leq i\leq n,1\leq j\leq m $. Let $\ggg^k$ be a variance reduced stochastic gradient estimator from Algorithm~\ref{alg:acc_stoch} and define
		\[
		\cL \eqdef \max \left\{ \frac{\Lloc}{n(1-\proby)}, \frac{\lambda}{n\proby} \right\}.
		\]
		Then, we have
		\begin{equation}\label{eq:exp:smooth_stoch}
		\E{\norm{\ggg^k - \nabla F(x^k)}^2} \leq 2\cLL D_F(w^k,x^k).
		\end{equation}
	\end{lemma}
	
	\begin{proof}

		\begin{eqnarray*}
			\E{\norm{\ggg^k - \nabla F(x^k)}^2} & = &
			\frac{1-p}{m} \sum_{j=1}^m \sum_{i=1}^n\norm{\frac{1}{1-p} \left( \nabla \flocc_{ij}(x^k)  - \nabla \flocc_{ij}(w^k) \right)- \left( \nabla F(x^k)  - \nabla F(w^k)\right)}^2 \\
			&& \qquad
			+
			p\norm{\frac{\lambda}{p} \left( \nabla \psi(x^k)  - \nabla \psi(w^k) \right)- \left( \nabla F(x^k)  - \nabla F(w^k)\right)}^2
			\\
			&\leq &
			\frac{1-p}{m} \sum_{j=1}^m \sum_{i=1}^n\norm{\frac{1}{1-p} \left( \nabla \flocc_{ij}(x^k)  - \nabla \flocc_{ij}(w^k) \right)}^2 \\
			&& \qquad
			+
			p\norm{\frac{\lambda}{p} \left( \nabla \psi(x^k)  - \nabla \psi(w^k) \right))}^2
			\\
			&=&
			\frac{1}{m(1-p)}\sum_{j=1}^m \sum_{i=1}^n\norm{\nabla \flocc_{ij}(x^k)  - \nabla \flocc_{ij}(w^k) }^2
			+
			\frac{\lambda^2}{p} \norm{\left( \nabla \psi(x^k)  - \nabla \psi(w^k) \right))}^2
			\\
			&\stackrel{(*)}{\leq}&
			\frac{2\Lloc}{n m(1-p)} \sum_{j=1}^m \sum_{i=1}^n D_{\flocc_{ij}}(w^k,x^k)
			+
			\frac{2\lambda^2}{np}D_\psi(w^k,x^k)
			\\
			&=&
			\frac{2\Lloc}{1-p}  D_{f}(w^k,x^k)
			+
			\frac{2\lambda^2}{np}D_\psi(w^k,x^k)
			\\
			\\
			&\leq&
			2\max \left\{ \frac{\Lloc}{n(1-p)}, \frac{\lambda}{np} \right\} D_F(w^k,x^k)
			\\
			&=& 2\cLL D_F(w^k,x^k).
		\end{eqnarray*}
		
		Above, $(*)$ holds since $\flocc$ is $\frac{\LLL}{n}$ smooth and $\psi$ is $\frac1n$ smooth~\cite{hanzely2020federated}.
		
	\end{proof}
	
	\begin{proposition} \label{prop:acc}
		Let $\flocc_{ij}$ be $L$ smooth and $\mu$ strongly convex for all $1\leq i\leq n,1\leq j\leq m $.
		Define the following Lyapunov function:
		\begin{eqnarray*}
			\Psi^k &\eqdef&  \norm{z^k - x^\star}^2 + \frac{2\gamma\beta}{\theta_1}\left[F(y^k) - F(x^\star)\right] + \frac{(2\theta_2 + \theta_1)\gamma\beta}{\theta_1\probx}\left[F(w^k) - F(x^\star)\right],
		\end{eqnarray*}
		and let
		\begin{eqnarray*}
			L_F &=&\frac1n( \lambda + \Lloc), \\
			\eta &=&  \frac14 \max\{L_F, \cL\}^{-1}, \\
			\theta_2 &=& \frac{\cL}{2\max\{L_F, \cL\}}, \\
			\gamma &=& \frac{1}{\max\{2\mu/n, 4\theta_1/\eta\}},\\
			\beta &=& 1 - \frac{\gamma\mu}{n} \; \mathrm{and} \\
			\theta_1 &=& \min\left\{\frac{1}{2},\sqrt{\frac{\eta\mu}{n} \max\left\{\frac{1}{2}, \frac{\theta_2}{\rho}\right\}}\right\} .
		\end{eqnarray*}
		Then the following inequality holds:
		\begin{equation*}
		\E{\Psi^{k+1}} \leq
		\left[1 -  \frac{1}{4}\min\left\{\probx, \sqrt{\frac{\mu}{2n\max\left\{L_{F}, \frac{\cL}{\rho}\right\}}} \right\} \right]\Psi^0.
		\end{equation*}
		As a consequence, iteration complexity of Algorithm~\ref{alg:acc_stoch} is
		\[
		\cO\left(\left(\frac{1}{\probx}   + \sqrt{\frac{ \max \left\{ \frac{\Lloc}{1-\proby}, \frac{\lambda}{\proby} \right\}}{\probx\mu}}  \right)\log\frac1\varepsilon\right).
		\]
	\end{proposition}
	At the same time, the communication complexity of {\tt AL2SGD+} is 
	\[
	\cO\left(  (\probx + \proby(1-\proby))\left( \frac{1}{\probx}   + \sqrt{\frac{ \max \left\{ \frac{\Lloc}{1-\proby}, \frac{\lambda}{\proby} \right\}}{\probx\mu}}\right)\log\frac1\varepsilon\right)\] and the local stochastic gradient complexity is
	\[
	\cO\left((
	\probx m +
	(1-\probx)
	)\left(\frac{1}{\probx}   + \sqrt{\frac{ \max \left\{ \frac{\Lloc}{1-\proby}, \frac{\lambda}{\proby} \right\}}{\probx\mu}}  \right)\log\frac1\varepsilon
	\right).
	\]
	\begin{proof}
		
		Note that {\tt AL2SGD+} is a special case of {\tt L-Katyusha} from~\cite{hanzely2020variance}.\footnote{
			Similarly, we could have applied different accelerated variance reduced method with importance sampling such as another version of {\tt L-Katyusha}~\cite{qian2019svrg}, for example.
		} In order to apply Theorem 4.1 therein directly, it suffices to notice that function $F$ is $L_F =\frac1n( \lambda + \Lloc)$ smooth and $\frac1n \mu$ strongly convex, and at the same time, thanks to Lemma~\ref{lem:es_stoch} we have
		\[
		\E{\norm{g^k - \nabla F(x^k)}^2} \leq 2\cL D_F(w^k,x^k).
		\]
		Connsequently, we immediately get the iteration complexity. The local stochastic gradient complexity of a single iteration of {\tt AL2SGD+} is  
		0 if $\xi = 1, \xi' = 0$, 1 if $\xi=0, \xi'=0$, $m$ if $\xi=0, \xi'=1$ and $m+1$ if $\xi=1, \xi'=1$. Thus, the total expected local stochastic gradient complexity is bounded by
		\[
		\cO\left((
		\probx m +
		(1-\probx)
		)\left(\frac{1}{\probx}   + \sqrt{\frac{ \max \left\{ \frac{\Lloc}{1-\proby}, \frac{\lambda}{\proby} \right\}}{\probx\mu}}  \right)\log\frac1\varepsilon
		\right)
		\]
		as desired. Next, the total communication complexity is bounded by the sum of the communication complexities coming from the full gradient computation (if statement that includes $\xi$) and the rest (if statement that includes $\xi'$). The former requires a communication if $\xi'=1$, the latter if two consecutive $\xi$-coin flips are different (see~\cite{hanzely2020federated}), yielding the expected total communication $\cO(\probx + \proby(1-\proby))$ per iteration.
		
	\end{proof}
	
	\subsubsection{Proof of Theorem~\ref{thm:a2}}
	
	For $\probx = \proby(1-\proby)$ and $\proby = \frac{\lambda}{\lambda+\Lloc}$, the total communication complexity of {\tt AL2SGD+} becomes
	\begin{eqnarray*}
		\cO\left(  (\probx + \proby(1-\proby))\left( \frac{1}{\probx}   + \sqrt{\frac{ \max \left\{ \frac{\Lloc}{1-\proby}, \frac{\lambda}{\proby} \right\}}{\probx\mu}}\right)\log\frac1\varepsilon\right)
		&=&
		\cO\left(  \sqrt{\frac{ \proby \Lloc + (1-\proby) \lambda}{\mu}}\log\frac1\varepsilon\right)
		\\
		&=&
		\cO\left(  \sqrt{\frac{ \Lloc \lambda}{(\Lloc+ \lambda)\mu}}\log\frac1\varepsilon\right)
		\\
		&=&
		\cO\left(  \sqrt{\frac{ \min\{ \Lloc , \lambda\}}{\mu}}\log\frac1\varepsilon\right)
	\end{eqnarray*}
	as desired. 
	
	The local stochastic gradient complexity for $\proby = \frac{\lambda}{\lambda+\Lloc}$ and $\rho = \frac{1}{m}$ is
	\begin{eqnarray*}
		&&
		\cO\left((
		\probx m +
		(1-\probx)
		)\left(\frac{1}{\probx}   + \sqrt{\frac{ \max \left\{ \frac{\Lloc}{1-\proby}, \frac{\lambda}{\proby} \right\}}{\probx\mu}}  \right)\log\frac1\varepsilon
		\right)
		\\
		&& \qquad =
		\cO\left( \left(
		m
		+ \sqrt{\frac{m (\Lloc + \lambda)}{\mu}}  \right)\log\frac1\varepsilon
		\right).
	\end{eqnarray*}
	
	\QED
	
	\clearpage
		
	\section{Related work on the lower complexity bounds}

	\paragraph{Related literature on the lower complexity bounds.} We distinguish two main lines of work on the lower complexity bounds related to our paper besides already mentioned works~\cite{scaman2018optimal, hendrikx2020optimal}. 
	
	The first direction focuses on the classical worst-case bounds for sequential optimization developed by Nemirovsky and Yudin~\cite{nemirovski1985optimal}. Their lower bound was further studied in~\cite{agarwal2009information,raginsky2011information, nguyen2019tight} using information theory. The nearly-tight lower bounds for deterministic non-Euclidean smooth convex functions were obtained in \cite{guzman2015lower}. A significant gap between the oracle complexities of deterministic and randomized algorithms for the finite-sum problem was shown in~\cite{woodworth2016tight}, improving upon prior works~\cite{agarwal2014lower, lan2018optimal}.
	
The second stream of work tries to answer how much a parallelism might improve upon a given oracle. This direction was, to best of our knowledge, first explored by the work of Nemirovski~\cite{nemirovski1994parallel} and gained a lots of traction decently~\cite{smith2017interaction, balkanski2018parallelization, woodworth2018graph, duchi2018minimax, diakonikolas2018lower} motivated by an increased interest in the applications in federated learning, local differential privacy, and adaptive data analysis.

	\begin{remark} 
Concurrently with our work,	a different variant of accelerated {\tt FedProx}---{\tt FedSplit}---was proposed in~\cite{pathak2020fedsplit}. There are several key differences between our work:
		i) While Algorithm~\ref{alg:fista_inex} is designed to tackle the problem~\eqref{eq:main}, {\tt FedSplit} is designed to tackle~\eqref{eq:fl_standard}. ii) The paper \cite{pathak2020fedsplit} does not argue about optimality of {\tt FedSplit}, while we do and iii) Iteration/communication complexity of {\tt FedSplit} is $\cO\left( \sqrt{\frac{L}{\mu}}\log \frac1\varepsilon\right)$ under $L$ smoothness of $f_1, \dots f_n$; such a rate can be achieved by a direct application of {\tt AGD}. At the same time, {\tt AGD} does not require solving the local subproblem each iteration, thus is better in this regard. However {\tt FedSplit} is a local algorithm to solve~\eqref{eq:fl_standard} with the correct fixed point, unlike other popular local algorithms.
	\end{remark}

\end{document}